\def\year{2015}
\documentclass[letterpaper]{article}
\usepackage{authblk}
\usepackage{graphicx} % more modern
\usepackage{subfigure} 
\usepackage[table,usenames,dvipsnames]{xcolor}
\usepackage{amsthm}
\usepackage{scrextend}
\usepackage{algorithm}
\usepackage{algorithmic}
\usepackage{hyphenat}
\usepackage{times}
\usepackage{balance}
\newtheorem{example}{Example}
\newtheorem{problem}{Problem}

\usepackage{array,booktabs}
\usepackage[utf8]{inputenc}
\usepackage{tikz}
\usetikzlibrary{fit,automata,positioning,shapes,shapes.multipart,
arrows,decorations.pathreplacing,matrix}
\usetikzlibrary{shapes.callouts}
\usetikzlibrary{decorations.text}
\tikzset{text/.default=}
\tikzset{node/.default=}
\usepackage{dcolumn}
\usepackage{amssymb}% http://ctan.org/pkg/amssymb
\usepackage{pifont}% http://ctan.org/pkg/pifont 
\newcommand{\cmark}{\textcolor{ForestGreen}{\ding{51}}}%
\newcommand{\xmark}{\textcolor{Maroon}{\ding{55}}}%
\newcolumntype{M}[1]{D{.}{.}{1.#1}}
\usepackage{amsmath}% http://ctan.org/pkg/amsmath
\DeclareMathOperator*{\argmax}{arg\,max}

\usepackage{pgfplots}
\usepackage{pgfkeys}
\pgfplotsset{compat=1.3}
\usetikzlibrary{patterns}
\usepackage{enumitem}
\usepackage[font={small},skip=1pt]{caption}

\usepackage{courier}
\usepackage{listings}

\newcommand{\rw}[1]{{\textcolor{black}{{\bf}#1}}}

\newtheorem{lemma}{Lemma}
\newtheorem{theorem}{Theorem}

\newcommand{\opt}{\textsc{Opt}\xspace}
\newcommand{\greedy}{\textsc{Greedy}\xspace}
\newcommand{\ig}{\text{IG}\xspace}

\usepackage{amsmath,amssymb,xfrac,xspace}
\newcommand{\Real}{{\mathbb{R}}}
\usepackage{multirow}
\usepackage{amssymb}
\usepackage{centernot}

\setitemize[0]{leftmargin=3mm}
\newenvironment{my_list}{
   \begin{itemize}
   \setlength{\itemsep}{2.0pt}
   \setlength{\parskip}{-2pt}
   \setlength{\parsep}{-2pt}
}{
   \end{itemize}
}

\newcommand{\smallsection}[1]{\vspace{1mm}\noindent\textbf{#1.}}	% define own new subsection type: noindent, bold
\usepackage[algo2e, vlined, ruled]{algorithm2e}	
\DontPrintSemicolon     		% does not print ";" at end of each line
\SetNlSty{}{}{}					% style of line numbers "\Setnlsty{small}{}{:}"
\SetAlgoCaptionSeparator{.}		% no dot after 'Algorithm 3' in caption
\SetAlgoSkip{smallskip}			% space before and after algorithm environment: medskip, smallskip
\let\oldnl\nl% Store \nl in \oldnl
\newcommand{\nonl}{\renewcommand{\nl}{\let\nl\oldnl}}% Remove line number for one line

\SetAlgoCaptionLayout{myCap}					% takeds only one command, hence above defined

\SetInd{1.2mm}{1.2mm}	% corrects for the twice usage of "\Indp"; otherwise vlines are wrong

\SetKwFor{ForAll}{forall}{do}{endfch}	% change from "forall the" to "forall"
\LinesNumbered
\usepackage{aaai}
\usepackage{times}
\usepackage{helvet}
\usepackage{courier}
\begin{document}
\renewcommand\Authfont{\large\bfseries}
\renewcommand\Affilfont{\small}
\renewcommand\Authsep{\quad}
\renewcommand\Authand{\quad}
\renewcommand\Authands{\quad}
% The file aaai.sty is the style file for AAAI Press 
% proceedings, working notes, and technical reports.
% 
\title{Crowd Access Path Optimization: Diversity Matters}
%\author{Besmira Nushi\\affil1 \and Adish Singla \quad Anja Gruenheid \quad Erfan Zamanian \authorcr Andreas Krause \quad Donald Kossmann}
\author[1]{Besmira Nushi}
\author[1]{Adish Singla}
\author[1]{Anja Gruenheid}
\author[2]{Erfan Zamanian} 
\author[1]{\authorcr Andreas Krause}
\author[1,3]{Donald Kossmann}

\affil[1]{ETH Zurich, Department of Computer Science, Switzerland}
\affil[2]{Brown University, Providence, USA }
\affil[3]{Microsoft Research, Redmond, WA, USA }
\maketitle
\begin{abstract}
\begin{quote}
Quality assurance is one the most important challenges in crowdsourcing.
Assigning tasks to several workers to increase quality through redundant answers can be expensive if asking homogeneous
sources. This limitation has been overlooked by current crowdsourcing platforms resulting therefore in costly solutions.
In order to achieve desirable cost-quality tradeoffs it is essential to apply efficient crowd access optimization
techniques. Our work argues that optimization needs to be aware of diversity and correlation of information within
groups of individuals so that crowdsourcing redundancy can be adequately planned beforehand. Based on this intuitive
idea, we introduce the Access Path Model (APM), a novel crowd model that leverages the notion of access paths as an
alternative way of retrieving information. APM aggregates answers ensuring high quality and meaningful confidence.
Moreover, we devise a greedy optimization algorithm for this model that finds a provably good approximate plan to access
the crowd. We evaluate our approach on three crowdsourced datasets that illustrate various aspects of the problem. Our
results show that the Access Path Model combined with greedy optimization is cost-efficient and practical to overcome
common difficulties in large-scale crowdsourcing like data sparsity and anonymity.
\end{quote} 
\end{abstract} 
%%%%%%%%%%%%%%%%%%%%%%%%%%%%%%% BEGIN INTRODUCTION SECTION %%%%%%%%%%%%%%%%%%%%%%%%%%%%%%%%%%%%%%%%%%%%%%%%%%%%%%%%%%%%%%%%%%%%%%%%%
\section{Introduction}
% Intro to crowdsourcing, how is it used in databases, which are the challenges
Crowdsourcing has attracted the interest of many research communities
such as database systems, machine learning, and human computer interaction
because it allows humans to collaboratively solve problems that are difficult to handle with machines only.
Two crucial challenges in crowdsourcing independent of the field of application are (i) \emph{quality assurance} and (ii) \emph{crowd access optimization}.
Quality assurance provides strategies that proactively plan and ensure the
quality of algorithms run on top of crowdsourced data. Crowd access optimization
then supports quality assurance by carefully selecting from a large pool the
crowd members to ask under limited budget or quality
constraints. In current crowdsourcing platforms, redundancy
(\emph{i.e.}~assigning the same task to multiple workers) is the most common and straightforward way to guarantee quality \cite{karger2011budget}.
Simple as it is, redundancy can be expensive if used without any target-oriented approach, especially if the errors of
workers show dependencies or are correlated.
Asking people whose answers are expected to converge to the same opinion is neither efficient nor insightful. 
For example, in a sentiment analysis task, one would prefer to consider opinions from different non-related groups of interests before forming a decision. 
This is the basis of the diversity principle introduced by \cite{surowiecki2005wisdom}. 
The principle states that the best answers are achieved from discussion and contradiction rather than agreement and consensus.

In this work, we incorporate the diversity principle in a novel crowd model, named \textbf{Access Path Model} (APM),
which seamlessly tackles quality assurance and crowd access optimization and is applicable in a wide range of use cases.
It explores crowd diversity not on the individual worker level but on the common dependencies of workers while performing a task. 
In this context, an \emph{access path} is a way of retrieving a piece of information from the crowd. 
The configuration of access paths can be based on various criteria depending on the task: 
(i) workers' demographics (\emph{e.g.} profession, group of interest, age) (ii)  the source of information or the tool
that is used to find the answer (\emph{e.g.} phone call vs. web page, Bing vs. Google) (iii) task design (\emph{e.g.}
time of completion, user interface) (iv) task decomposition (\emph{e.g.} part of the answers, features).
% example
\begin{example}
\label{introduction.example}
Peter and Aanya natively speak two different languages which they would like to teach to their young children. 
At the same time, they are concerned how this multilingual environment affects the learning abilities of their children. 
More specifically, they want to answer the question ``Does raising children bilingually cause language delay?''. 
To resolve their problem, they can ask three different groups of people (access paths):
\end{example}
\begin{table}[ht!]
	\footnotesize
	\centering
	\begin{tabular}{lcc}
		\midrule
		\bf{Access Path}				& \bf{Error rate}	& \bf{Cost}	\\
		\midrule
		Pediatricians				& 10\%				& \$20		\\ \hline 
		Logopedists					& 15\%				& \$15		\\\hline
		Other parents				& 25\%				& \$10		\\\hline
	\end{tabular}
	\caption{Access path configuration for Example 1}
\end{table}
Figure~\ref{fig:IntroExample} illustrates the given situation with respect to
the Access Path Model. In this example, each of the groups approaches the
problem from a different perspective and has different associated error rates and costs.
Considering that Peter and Aanya have a limited budget to spend and can ask more than one person on the same access path, they are interested in finding the optimal combination of access paths that will give them the most insightful information for their budget constraints. 
Throughout this paper, a combination of access paths will
be referred to as an \emph{access plan} and it defines how many different people
to ask on each available access path. Our model aims at helping general
requesters in crowdsourcing platforms to find optimal access plans and
appropriately aggregate the collected data. Results from
 experiments on real-world crowdsourcing show that a pre-planned combination
of diverse access paths indeed overperforms pure (\emph{i.e.} single access path) access plans, random selection, and equal
distribution of budget across access paths. The main finding is that diversity is a powerful mean that matters for
quality assurance.
\begin{figure}[t]
\centering
\resizebox{\columnwidth}{!}{
\begin{tikzpicture}
\tikzstyle{text}=[draw=none,fill=none]
\tikzstyle{main}=[circle, minimum size = 7mm, thick, draw =black!80, node distance = 16mm]
\tikzstyle{ap}=[rectangle, minimum width = 27mm,minimum height = 5mm, rounded corners, thick, draw =black!80, node distance = 16mm]
\tikzstyle{connect}=[-latex, thick]
\tikzstyle{box}=[rectangle, draw=black!100,  dotted]
\tikzset{callout/.style=
 {rectangle callout,
  fill=gray!20,
  callout absolute pointer={#1},
  at={#1},
  above=1cm,
  draw  
  }}

  \node (Y)  {\includegraphics[width=1.5cm]{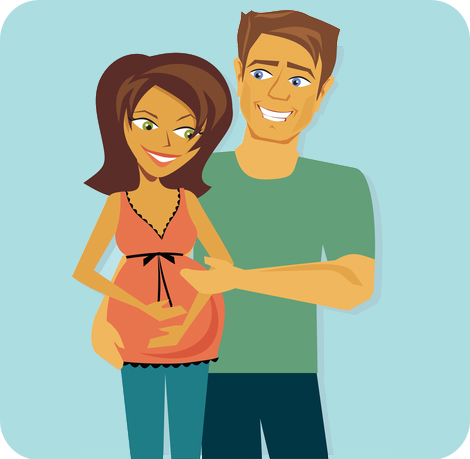}};
  \node[ap, fill = black!10] (Z2) [below=0.4 cm of Y] {Logopedists};
  \node[ap, fill = black!10] (Z1) [left=1.2 cm of Z2] {Pediatricians};
  \node[ap, fill = black!10] (Z3) [right=1.4 cm of Z2] {Parents};
  
  \node[text] (X12) [below= 0.9cm of Z1,inner sep=0pt,label=below:$\ldots$] {};
  \node (X11) [left=0.25cm of X12,inner sep=0pt] {\includegraphics[width=1.2cm]{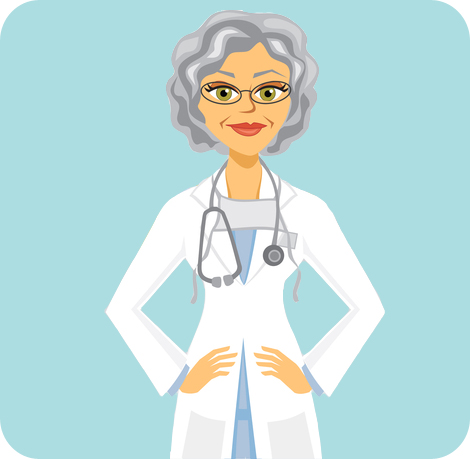}};
  \node (X1N) [right=0.25cm of X12,inner sep=0pt] {\includegraphics[width=1.2cm]{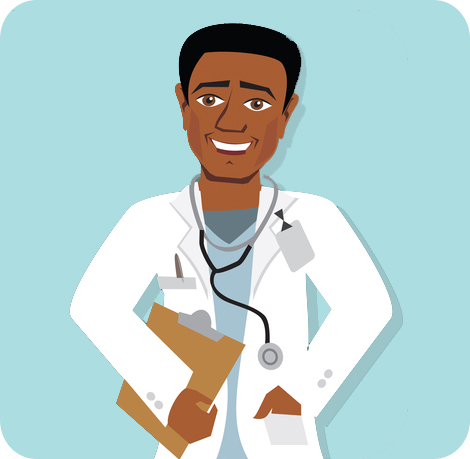}};

  \node[text] (X22) [below= 0.9cm of Z2,inner sep=0pt,label=below:$\ldots$] {};
  \node (X21) [left=0.25cm of X22,inner sep=0pt] {\includegraphics[width=1.2cm]{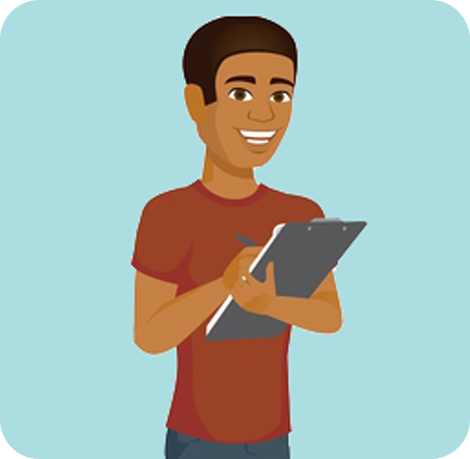}};
  \node (X2N) [right=0.25cm of X22,inner sep=0pt] {\includegraphics[width=1.2cm]{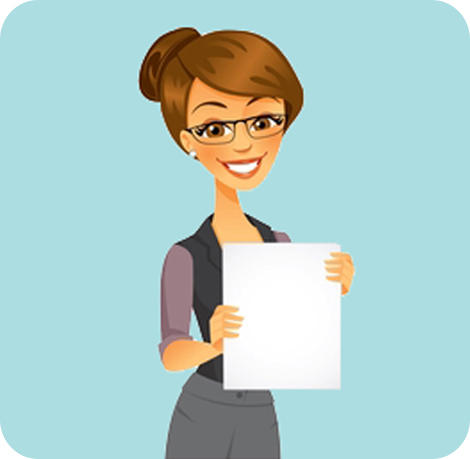}};

  \node[text] (X32) [below= 0.9cm of Z3,inner sep=0pt,label=below:$\ldots$] {};
  \node (X31) [left=0.25cm of X32,inner sep=0pt] {\includegraphics[width=1.2cm]{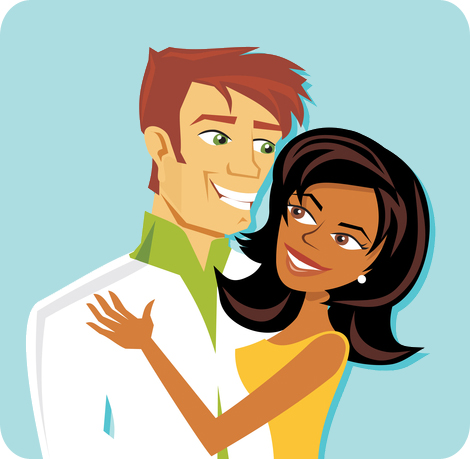}};
  \node (X3N) [right=0.25cm of X32,inner sep=0pt] {\includegraphics[width=1.2cm]{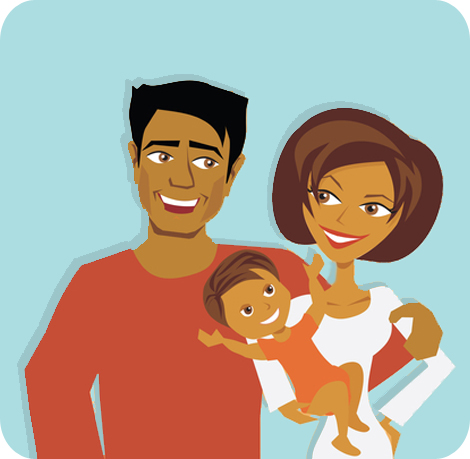}};
  \node[callout={(1,.5)}] (callout) [above right=-1.0cm and 0.2cm of Y] {\small
  \begin{tabular}{c}
  Does raising children\\
  bilingually cause\\
  language delay?\\
  \end{tabular}};
  \node[text] (q1) [above= 1.4cm of X1N,inner sep=0pt] {\bf ?};
  \node[text] (q2) [above right= 1.6cm and 0.1cm of X22,inner sep=0pt] {\bf ?};
  \node[text] (q3) [above= 1.4cm of X31,inner sep=0pt] {\bf ?};
  
  \path (Y) edge [connect] (Z1)
        (Y) edge [connect] (Z2)
        (Y) edge [connect] (Z3)
        (Z1) edge [connect] (X11)
        (Z1) edge [connect] (X12)
        (Z1) edge [connect] (X1N)
        (Z2) edge [connect] (X21)
        (Z2) edge [connect] (X22)
        (Z2) edge [connect] (X2N)
        (Z3) edge [connect] (X31)
        (Z3) edge [connect] (X32)
        (Z3) edge [connect] (X3N);
\end{tikzpicture}
}
\caption{APM for crowdsourcing a medical question}
\label{fig:IntroExample} 
\end{figure}
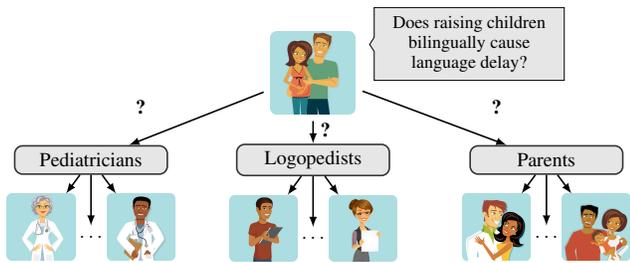
\subsection{Contributions}
% contrast to previous work
Previous work on quality assurance and crowd access optimization focuses on two different approaches: majority-based
strategies and individual models. Majority voting is oblivious to personal characteristics of crowd workers and is
therefore limited in terms of optimization. Individual models instead base their decisions on the respective performance
of each worker targeting those with the best accuracy \cite{dawid1979maximum,whitehill2009whose}.
These models are useful for spam detection and pricing schemes but do not guarantee answer diversity and might fall into
partial consensus traps. 

As outlined in Table~\ref{table:apm_comparison}, the APM is a middle-ground solution between these two choices and offers several advantages.
First, it is aware of answer diversity which is particularly important for requests without an established ground truth.
Second, since it manages group-based answer correlations and dependencies, it facilitates efficient optimization of
redundancy. Third, the APM is a practical model for current crowdsourcing marketplaces where due to competition the
availability of a particular person is never guaranteed or authorships may be hidden for privacy reasons. Last, its
predictions are mapped to meaningful confidence levels which can simplify the interpretation of results.

\begin{table}[h!]
	\footnotesize
	\centering
	\begin{tabular}{|m{18mm}|c|c|c|}	
		\cline{2-4}
		 \multicolumn{1}{c|}{}		& Majority		 	& Individual			& Access Path \\ 
		 \multicolumn{1}{c|}{}		& Voting 		 	& Models				& Model \\ \hline
		Diversity awareness 		& \Large{\xmark} 	& \Large{\cmark}		& \Large{\cmark}		\\ \hline
		Cost-efficient optimization & \Large{\xmark}	& \Large{\xmark} 		& \Large{\cmark}		\\\hline 
		Sparsity  Anonymity   		& \Large{\cmark}	& \Large{\xmark} 		& \Large{\cmark}		\\\hline 
		Meaningful confidence 		& \Large{\cmark}	& \Large{\xmark} 		& \Large{\cmark}		\\\hline 
	\end{tabular}
	\caption{Comparison of APM with current approaches.}
	\label{table:apm_comparison}
\end{table}
In summary, this work makes the following contributions:
\begin{my_list}
\item \textbf{Modeling the crowd for quality assurance.} We design the Access
Path Model as a Bayesian Network that through the usage of latent variables is able to capture and utilize
crowd diversity from a non-individual
point of view. The APM can be applied even if the data is sparse and
crowd workers are anonymous.
\item \textbf{Crowd access optimization.} We use an information-theoretic objective for crowd access optimization. We
prove that our objective is submodular, allowing us to adopt efficient greedy algorithms with strong
guarantees.
\item \textbf{Real-world experiments.} Our extensive experiments cover three different domains: 
Answering medical questions, sport events prediction and bird species classification. 
We compare our model and optimization scheme with state of the art techniques and show that it makes robust predictions with lower cost. 
\end{my_list}
%%%%%%%%%%%%%%%%%%%%%%%%%%%%%%% END INTRODUCTION SECTION %%%%%%%%%%%%%%%%%%%%%%%%%%%%%%%%%%%%%%%%%%%%%%%%%%%%%%%%%%%%%%%%%%%%%%%%%
%%%%%%%%%%%%%%%%%%%%%%%%%%%%%%% BEGIN PROBLEM SECTION %%%%%%%%%%%%%%%%%%%%%%%%%%%%%%%%%%%%%%%%%%%%%%%%%%%%%%%%%%%%%%%%%%%%%%%%%
% !TEX root =  ../document.tex
\section{Problem Statement}
\label{sec.problemstatement}
In this work, we identify and address two closely related problems:
(1) modeling and aggregating diverse crowd answers which we call the {\it crowdsourced
predictions problem}, and (2) optimizing the budget distribution for better quality referred to as {\it access path
selection problem}.
\begin{problem}[\textsc{Crowdsourced Predictions}]
\label{problem1}
Given a task represented by a random variable $Y$, and a set of answers from $W$ workers represented by
random variables $X_1, \ldots, X_W$, the crowdsourced prediction problem is to find a high-quality prediction of the
outcome of task $Y$ by aggregating these votes.
\end{problem}
\smallsection{\bfseries Quality criteria} A high-quality prediction is not only accurate but should also be linked
to a meaningful confidence score which is formally defined as the likelihood of the prediction to be correct. This
property simplifies the interpretation of predictions coming from a probabilistic model. For example, if a doctor wants to know
whether a particular medicine can positively affect the improvement of a disease condition, providing a raw
\emph{yes/no} result answer is not sufficiently informative. Instead, it is much more
useful to associate the answer with a trustable confidence score. 

\smallsection {\bfseries Requirements and challenges} To provide high quality
predictions, it is essential to precisely represent the crowd. The main
aspects to be represented are (i) the conditional dependence of worker answers within access paths given the task and
(ii) the conditional independence of worker answers across access paths. As we will show in this paper, modeling such dependencies is also crucial for efficient optimization.
Another realistic requirement concerns the support for data \emph{sparsity} and \emph{anonymity}. Data sparsity is common in crowdsourcing \cite{venanzi2014community} and occurs when the number of tasks that workers solve is not sufficient
to estimate their errors which can negatively affect quality. In other cases, the identity of workers is not available,
but it is required to make good predictions based on non-anonymized features. 
\begin{problem}[\textsc{Access Path Selection}]
\label{problem2}
Given a task represented by a random variable $Y$, that can be solved by the crowd following $N$ different access paths
denoted with the random variables $Z_1, \ldots, Z_N$, using a maximum budget $B$, the access path selection problem is
to find the best possible access plan $S_{best}$ that leads to a high-quality prediction of the outcome of task $Y$.
\end{problem}
An access plan defines how many different people are chosen to complete the task from each access path.
In Example~\ref{introduction.example}, we will ask one pediatrician, two logopedists and three different parents if the
access plan is $S=[1, 2, 3]$. Each access plan is associated with a cost $c(S)$ and quality $q(S)$.
For example, ~$c(S) = \sum_{i=1}^{3} c_i \cdot S[i] = \$80$ where $c_i$ is the cost of getting one single answer through
access path $Z_i$. In these terms, the access path selection problem can be generally formulated as:
\begin{equation}
\label{eq:apselectionproblem}
S_{best} = \argmax_{S \in \mathcal{S}} q(S) \text{ \small{s.t.} }
\sum_{i=1}^{N} c_i \cdot S[i] \leq B
\end{equation}
This knapsack maximization problem is NP-Hard even for submodular functions \cite{1998-_feige_threshold-of-ln-n}.
Hence, designing bounded and efficient approximation schemes is
useful for realistic crowd access optimization.
%%%%%%%%%%%%%%%%%%%%%%%%%%%%%%% END PROBLEM SECTION %%%%%%%%%%%%%%%%%%%%%%%%%%%%%%%%%%%%%%%%%%%%%%%%%%%%%%%%%%%%%%%%%%%%%%%%%
%%%%%%%%%%%%%%%%%%%%%%%%%%%%%%% BEGIN MODEL SECTION %%%%%%%%%%%%%%%%%%%%%%%%%%%%%%%%%%%%%%%%%%%%%%%%%%%%%%%%%%%%%%%%%%%%%%%%%
% !TEX root =  ../document.tex
\section{Access Path Model}
\label{sec.model}
The crowd model presented in this section aims at fulfilling the requirements specified in the definition of
Problem~\ref{problem1} (\textsc{Crowdsourced Prediction}) and enables our method
to learn the error rates from historical data and then accordingly aggregate
worker votes.\label{subsec:baseline}
\subsection{Access Path Design}
\label{sec:design}
Due to the variety of problems possible to
crowdsource, an important step concerns the design of access paths. The
access path notion is a broad concept that can accommodate various situations
and may take different shapes depending on the task. Below we describe a list of viable configurations that can be easily applied
in current platforms.
\begin{my_list}
\item {\bfseries Demographic groups.} Common demographic characteristics
(location, gender, age) can establish strong statistical
dependencies of workers' answers \cite{kazai2012face}. Such groups are
particularly diverse for problems like sentiment analysis or product evaluation
and can be retrieved from crowdsourcing platforms as part of the task, worker
information, or qualification tests.
\item {\bfseries Information sources.} For data collection
and integration tasks, the data source being used to deduplicate or match
records (addresses, business names \emph{etc.}) is the primary cause of error or
accuracy \cite{pochampally2014fusing}.
\item {\bfseries Task design.} In other cases, the answer of a worker may
be psychologically affected by the user interface design.
For instance, in crowdsourced sorting, a worker may rate
the same product differently depending on the scaling system (stars, 1-10
\emph{etc.}) or other products that are part of the same
batch \cite{parameswaran2014optimal}.
\item {\bfseries Task decomposition.} Often, complicated problems are
decomposed into smaller ones. Each subtask type can serve as an access path. For
instance, in the bird classification task that we study later in our experiments, workers can resolve separate features
of the bird (\emph{i.e.} color, beak shape \emph{etc.}) rather than its category.
\end{my_list}
In these scenarios, the access path definition
natively comes with the problem or the task design. However, there are scenarios where the structure is
not as evident or more than one grouping is applicable. Helpful tools in this regard include graphical model
structure learning based on conditional
independence tests \cite{de2006scoring} and information-theoretic group selection \cite{wisdomminority}.

\smallsection{Architectural implications} We envision access path design as
part of the quality assurance and control module for new crowdsourcing
frameworks or, in our case, as part of the query engine in a crowdsourced
database \cite{franklin2011crowddb}.
In the latter context, the notion of access paths is one of the main pillars in query optimization for traditional databases
\cite{selinger1979access} where access path selection (\emph{e.g.} sequential
scan or index) has significant impact on the query response time. In addition,
in a crowdsourced database the access path selection also affects the quality
of query results. In such an architecture, the query optimizer is responsible
for (i) determining the optimal combination of access paths as shown in the
following section, and (ii) forwarding the design to the UI creation. The query
executor then collects the data from the crowd and aggregates it through the
probabilistic inference over the APM.
\subsection{Alternative models}
\label{sec.stateoftheart}
Before describing the structure of the Access Path Model, we first have a look at other alternative models and their
behavior with respect to quality assurance. Table~\ref{accronym.table} specifies the meaning of each symbol as used
throughout this paper.

\smallsection{Majority Vote (\textsc{MV})} Being the simplest of the models and
also the most popular one, majority voting is able to produce fairly good
results if the crowdsourcing redundancy is sufficient. Nevertheless,
majority voting considers all votes as equal with respect to quality and
can not be integrated with any optimization scheme other than random selection.

\smallsection{Na\"{i}ve Bayes Individual (\textsc{NBI})} This model assigns
individual error rates to each worker and uses them to weigh the incoming votes and form a decision (Figure~\ref{fig:NBI}). In cases when
the ground truth is unknown, the error estimation is carried out through an EM Algorithm as proposed
by \cite{dawid1979maximum}. Aggregation (\emph{i.e.}
selecting the best prediction) is then performed through Bayesian inference. For
example, for a set of votes $x_t$ coming from $W$ different workers $X_1, \dots, X_W$ the most likely outcome among all
candidate outcomes $y_c$ is computed as $\text{prediction} = \argmax_{y_c \in Y} {p(y_c|x_t)}$,
whereas the joint probability of a candidate answer $y_c$ and the votes $x_t$ is:
\begin{gather}
p(y_c,x_t) = p(y)\prod_{w=1}^{W}p(x_{wt}|y_c)
\label{eq:nbiinf}
\end{gather}
The quality of predictions for this model highly depends on the assumption that each worker has solved a fairly
sufficient number of tasks. This assumption
generally does not hold for open crowdsourcing markets where stable participation of workers is not guaranteed. As we
show in the experimental evaluation, this is harmful not only for estimating the error rates but also for crowd access
optimization because access plans might not be imlplementable or have a high response time.
Furthermore, even in cases of fully committed workers, NBI does not provide the proper logistics to optimize the
budget distribution since it does not capture the shared dependencies between the workers. Last, due to the Na\"{i}ve Bayes
inference which assumes conditional independence between each pair of workers, predictions of this model are generally
overconfident.
\begin{table}[t]
\footnotesize
\centering
	\begin{tabular}{ll}
		\hline
		\bf{Symbol}				& \bf{Description}	\\ \hline
		$Y$						& random variable of the crowdsourced task	\\ \hline
		$X_w$					& random variable of worker $w$ 		\\\hline
		$W$						& number of workers		\\\hline		
		$Z_i$					& latent random variable of access path $i$		\\\hline
		$X_{ij}$				& random variable of worker $j$ in access path $i$	\\\hline
		$N$						& number of access paths		\\\hline
		$B$						& budget constraint		\\\hline
		$S$						& access plan		\\\hline
		$S[i]$					& no. of votes from access path $i$ in plan $S$		\\\hline		
		$c_i$					& cost of access path $i$		\\\hline
		$D$						& training dataset		\\\hline
		%$K$						& number of samples in the training dataset		\\\hline
		$s<y,x>$				& instance of task sample in a dataset 	\\\hline
		$\theta$				& parameters of the Access Path Model		\\\hline
	\end{tabular}
	\caption{Symbol description}
	\label{accronym.table}
\end{table}

\subsection{Access Path based models}
Access Path based models group the answers of the crowd according to the
access path they originate from. We first describe a simple Na\"{i}ve Bayes
version of such a model and then elaborate on the final design of the APM.

\smallsection{Na\"{i}ve Bayes for Access Paths (NBAP)} 
For correcting the effects of non-stable participation of individual workers we first consider another alternative,
similar to our original model, presented in Figure~\ref{fig:NBAP}.
The votes of the workers here are grouped according to the access path. For inference purposes then, each vote $x_{ij}$
is weighed with the average error rate $\theta_i$ of the access path it comes
from. In other words, it is assumed that all workers within the same access path
share the same error rate. As a result, all votes belonging to the same access path behave as a single random
variable, which  enables the model to support highly sparse data. Yet, due to the similarity with NBI and all Na\"{i}ve
Bayes classifiers, NBAP cannot make predictions with meaningful confidence especially when there exists a large number of access paths. 

\smallsection{Access Path Model overview} Based on the analysis of previous
models, we propose the Access Path Model as presented in Figure~\ref{fig:NetExample}, which shows an instantiation for three access paths.
We design the triple $<$task, access path, worker$>$ as a hierarchical Bayesian
Network in three layers. 
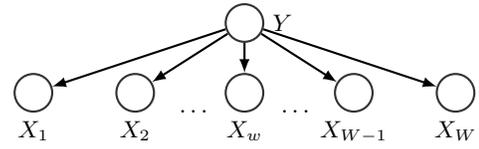
\begin{figure}[t]
\centering
\begin{tikzpicture}
\footnotesize
\tikzstyle{text}=[draw=none,fill=none]
\tikzstyle{main}=[circle, minimum size = 5mm, thick, draw =black!80, node distance = 16mm]
\tikzstyle{connect}=[-latex, thick]
\tikzstyle{box}=[rectangle, draw=black!100,  dotted]
  \node[main] (Y) [label=right:$Y$] { };

  \node[main] (X22) [below=0.4 cm of Y,label=below:$X_w$] { };
  \node[text] (X21) [left=0.3 cm of X22,label=below:$\ldots$] { };
  \node[text] (X2N) [right=0.3 of X22,label=below:$\ldots$] { };

  \node[main] (X1N) [left=0.4 cm of X21,label=below:$X_{2}$] { };
  \node[text] (X12) [left=0.3 cm of X1N,label=below:$$] { };
  \node[main] (X11) [left=0.3 cm of X12,label=below:$X_{1}$] { };

  \node[main] (X31) [right=0.4 cm of X2N,label=below:$X_{W-1}$] { };
  \node[text] (X32) [right=0.3 cm of X31,label=below:$$] { };
  \node[main] (X3N) [right=0.3 cm of X32,label=below:$X_{W}$] { };

  \path
        (Y) edge [connect] (X11)
        (Y) edge [connect] (X1N)
        (Y) edge [connect] (X22)
        (Y) edge [connect] (X31)
        (Y) edge [connect] (X3N);
\end{tikzpicture}
\caption{Na\"{i}ve Bayes Individual - NBI.}
\label{fig:NBI}
\end{figure}
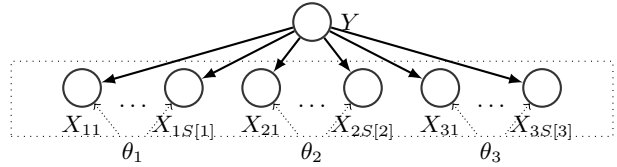
\begin{figure}[t]
\centering
\begin{tikzpicture}
\footnotesize
\tikzstyle{text}=[draw=none,fill=none]
\tikzstyle{main}=[circle, minimum size = 5mm, thick, draw =black!80, node distance = 16mm]
\tikzstyle{connect}=[-latex, thick]
\tikzstyle{box}=[rectangle, draw=black!100,  dotted]
  \node[main] (Y) [label=right:$Y$] { };

  \node[text] (X22) [below=0.5 cm of Y,label=below:$\ldots$] { };
  \node[main] (X21) [left=0.3 cm of X22,label=below:$X_{21}$] { };
  \node[main] (X2N) [right=0.3 of X22,label=below:$X_{2S[2]}$] { };

  \node[main] (X1N) [left=0.5 cm of X21,label=below:$X_{1S[1]}$] { };
  \node[text] (X12) [left=0.3 cm of X1N,label=below:$\ldots$] { };
  \node[main] (X11) [left=0.3 cm of X12,label=below:$X_{11}$] { };

  \node[main] (X31) [right=0.5 cm of X2N,label=below:$X_{31}$] { };
  \node[text] (X32) [right=0.3 cm of X31,label=below:$\ldots$] { };
  \node[main] (X3N) [right=0.3 cm of X32,label=below:$X_{3S[3]}$] { };

  \node[text] (th2) [below=0.5 cm of X22, font=\small] { $\theta_2$};
  \node[text] (th1) [below=0.5 cm of X12, font=\small] {$\theta_1$ };
  \node[text] (th3) [below=0.5 cm of X32, font=\small] { $\theta_3$};

 \node[box] (W) [below=0.25 cm of Y, label=right:, minimum width = 80 mm, minimum
 height = 10 mm]{};

  \path
        (Y) edge [connect] (X11)
        (Y) edge [connect] (X1N)
        (Y) edge [connect] (X21)
        (Y) edge [connect] (X2N)
        (Y) edge [connect] (X31)
        (Y) edge [connect] (X3N);
        
 \draw [densely dotted] (th1) [->] -- (X11);
 \draw [densely dotted] (th1) [->] -- (X1N);
 \draw [densely dotted] (th2) [->] -- (X21);
 \draw [densely dotted] (th2) [->] -- (X2N); 
 \draw [densely dotted] (th3) [->] -- (X31);
 \draw [densely dotted] (th3) [->] -- (X3N); 
\end{tikzpicture}
\caption{Na\"{i}ve Bayes Model for Access Paths - NBAP.}
\label{fig:NBAP}
\end{figure}
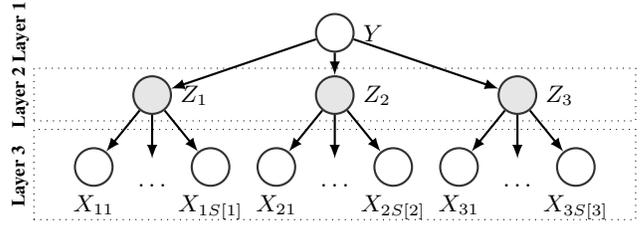
\begin{figure}[t!]
\centering
\begin{tikzpicture}
\footnotesize
\tikzstyle{text}=[draw=none,fill=none]
\tikzstyle{main}=[circle, minimum size = 5mm, thick, draw =black!80, node distance = 16mm]
\tikzstyle{connect}=[-latex, thick]
\tikzstyle{box}=[rectangle, draw=black!100,  dotted]
  \node[main] (Y) [label=right:$Y$] {};
  \node[main, fill = black!10] (Z2) [below=0.3 cm of Y,label=right:$Z_2$] {};
  \node[main, fill = black!10] (Z1) [left=1.9 cm of Z2,label=right:$Z_1$] {};
  \node[main, fill = black!10] (Z3) [right=1.9 cm of Z2,label=right:$Z_3$] {};
  
  \node[text] (X12) [below=0.58 cm of Z1,label=below:$\ldots$] {};
  \node[main] (X11) [left=0.4 cm of X12,label=below:$X_{11}$] {};
  \node[main] (X1N) [right=0.4 of X12,label=below:$X_{1S[1]}$] {};

  \node[text] (X22) [below=0.58 cm of Z2,label=below:$\ldots$] {};
  \node[main] (X21) [left=0.4 cm of X22,label=below:$X_{21}$] {};
  \node[main] (X2N) [right=0.4 of X22,label=below:$X_{2S[2]}$] {};

  \node[text] (X32) [below=0.58 cm of Z3,label=below:$\ldots$] {};
  \node[main] (X31) [left=0.4 cm of X32,label=below:$X_{31}$] {};
  \node[main] (X3N) [right=0.4 of X32,label=below:$X_{3S[3]}$] {};
  
% \node[box] (T) [label=right:, minimum width = 15 mm, minimum height = 10 mm]{};
 \node[box] (A) [below=0.2 cm of Y, label=right:, minimum width = 80 mm,
 minimum height = 7 mm]{}; 
 \node[box] (W) [below=0.1 cm of A, label=right:, minimum width = 80 mm,
 minimum height = 12 mm]{};
  
  \path (Y) edge [connect] (Z1)
        (Y) edge [connect] (Z2)
        (Y) edge [connect] (Z3)
        (Z1) edge [connect] (X11)
        (Z1) edge [connect] (X12)
        (Z1) edge [connect] (X1N)
        (Z2) edge [connect] (X21)
        (Z2) edge [connect] (X22)
        (Z2) edge [connect] (X2N)
        (Z3) edge [connect] (X31)
        (Z3) edge [connect] (X32)
        (Z3) edge [connect] (X3N);
  
  \node[text](L3)[below left=-1.10cm and 0.4cm of W,
  font=\scriptsize,rotate=90] {\bfseries{Layer 3}};
   \node[text](L2)[above right=0.02cm and 0.45cm of L3,
  font=\scriptsize,rotate=90] {\bfseries{Layer 2}};
  \node[text](L1)[above right=-0.17cm and 0.45cm of L2,
  font=\scriptsize,rotate=90] {\bfseries{Layer 1}};
\end{tikzpicture}
\caption{Bayesian Network Model for Access Paths - APM.}
\label{fig:NetExample}
\end{figure}

\noindent {\bfseries Layer 1.} Variable $Y$ in the root of the model represents
the random variable modeling the real outcome of the task. 

\noindent {\bfseries Layer 2.} This layer contains the random variables modeling the access paths $Z_1, Z_2, Z_3$. Each
access path is represented as a latent variable, since its values are not observable. Due to the tree structure, every
pair of access paths is conditionally independent given $Y$ while the workers that belong to the same access path are
not. The conditional independence is the key of representing diversity by implementing therefore various probabilistic
channels. Their purpose is to distinguish the information that can be obtained from the workers from the one that comes
from the access path.

Such enhanced expressiveness of this auxiliary layer over the previously described NBAP model avoids
overconfident predictions as follows. Whenever a new prediction is made, the amount of confidence that identical answers from different
workers in the same access path can bring is first blocked by the access path usage (\emph{i.e.} the latent variable). If the number of
agreeing workers within the same access path increases, confidence increases as well but not at the same rate as it
happens with NBI. Additional workers contribute only with their own signal, while the access path
signal has already been taken into consideration. In terms of optimization, this
property of the APM makes a good motivation for combining various access paths within the same plan.

\noindent {\bfseries Layer 3.} The lowest layer contains the random
variables $X$  modeling the votes of the workers grouped by the access path they
are following. For example, $X_{ij}$ is the $j$-th worker on the $i$-th
access path. The incoming edges represent the error rates of workers conditioned by their access paths.

\smallsection{Parameter learning} 
The purpose of the training stage is to learn the parameters of the model, \emph{i.e.}
the conditional probability of each variable with respect to its parents that are graphically represented by the network
edges in Figure~\ref{fig:NetExample}. We will refer to the set of all model
parameters as $\theta$. More specifically, $\theta_{Z_i|Y}$ represents the
table of conditional error probabilities for the $i$-th access path given the
task $Y$, and $\theta_{X_{ij}|Z_i}$ represents the table of conditional
error probabilities for the $j$-th worker given the $i$-th access path.

For a dataset $D$ with historical data of the same type of task, the
parameter learning stage finds the maximum likelihood estimate
$\theta_{\mathit{MLE}} = \argmax_{\theta} p(D|\theta)$. According to our model, the joint probability of a sample $s_k$
factorizes as:
\begin{equation}
p(s_k|\theta) = p(y_k|\theta) \prod_{i=1}^{N}\Big( p(z_{ik}|y_k,\theta)\prod_{j=1}^{S_{k}[i]}p(x_{ijk}|z_{ik},\theta) \Big)
\end{equation}
where $S_{k}[i]$ is the number of votes in access path $Z_i$ for the sample.
As the access path variables $Z_i$ are not
observable, we apply an Expectation Maximization (EM) algorithm \cite{dempster1977maximum} to find the best parameters. 
Notice that applying EM for the network model in Figure~\ref{fig:NetExample} will learn the parameters for each
worker in the crowd. This scheme works if the set of workers involved in the task is sufficiently stable
to provide enough samples for computing their error rates 
(\emph{i.e.} $\theta_{X_{ij}|Z_i}$) and if the worker id is not hidden.
As in many of the crowdsourcing applications (as well as in our experiments) this is not always the
case, we share the parameters of all workers within an access path.
This enables us to later apply on the model an optimization scheme agnostic
about the identity of workers.
The generalization is optional for the APM and obligatory for NBAP. 

\smallsection{Training cost analysis}
The amount of data needed to train the APM is significantly lower than
what individual models require which results in a faster learning process.
The reason is that the APM can benefit even from infrequent participation of
individuals $X_{ij}$ to estimate accurate error rates for access paths $Z_i$. Moreover, sharing the parameters of
workers in the same access path reduces the number of parameters to learn from $W$ for individual models to $2N$ for
the APM which is typically several orders of magnitude lower.

\smallsection{Inference} After parameter learning, the model is used to infer
the outcome of a task using the available votes on each access path. As in previous models, the inference step computes the likelihood of each candidate outcome $y_c \in Y$
given the votes in the test sample $x_t$ and chooses the most likely candidate
as $\text{prediction} = \argmax_{y_c \in Y} {p(y_c|x_t)}$. As the test samples
contain only the values for the variables $X$, the joint probability between the candidate outcome and the test sample is
computed by marginalizing over all possible values of $Z_i$ (Eq. \ref{eq:marginalization}). For a fixed cardinality of $Z_i$, the complexity of inferring the most
likely prediction is $\mathcal{O}(NM)$.
\begin{equation}
\label{eq:marginalization}
p(y_c,x_t) = p(y_c) \prod_{i=1}^{N} \Big( \sum_{z \in \{0,1\}} p(z|y_c)
\prod_{j=1}^{S_{t}[i]} p(x_{ijt}|z)\Big)
\end{equation}
The confidence of the prediction maps to the likelihood that
the prediction is accurate $p(\text{prediction}|x_t)$. Marginalization in
Equation \ref{eq:marginalization} is the technical step that avoids overconfidence by smoothly blocking the confidence increase when similar answers from the same access path are observed.
%%%%%%%%%%%%%%%%%%%%%%%%%%%%%%% END MODEL SECTION %%%%%%%%%%%%%%%%%%%%%%%%%%%%%%%%%%%%%%%%%%%%%%%%%%%%%%%%%%%%%%%%%%%%%%%%%
%%%%%%%%%%%%%%%%%%%%%%%%%%%%%%% BEGIN OPTIMIZATION SECTION %%%%%%%%%%%%%%%%%%%%%%%%%%%%%%%%%%%%%%%%%%%%%%%%%%%%%%%%%%%%%%%%%%%%%%%%%
\section{Crowd Access Optimization}
\label{sec.accesspathselection}
Crowd access optimization is crucial for both paid and non-paid of crowdsourcing.
While in paid platforms the goal is to acquire the best quality for the given monetary budget, in
non-paid applications the necessity for optimization comes from the fact that
highly redundant accesses might decrease user satisfaction and increase latency.
In this section, we describe how to estimate the quality of access plans and
how to choose the plan with the best expected quality.
\subsection{Information Gain as a measure of quality}
The first step of crowd access optimization is estimating the quality of access
plans before they are executed. One attempt might be to quantify the {\em accuracy}
of individual access paths in isolation, and choose an objective function that prefers the selection of more accurate
access paths. However, due to statistical dependencies of responses within an access path (\emph{e.g.}, correlated
errors in the workers' responses), there is diminishing returns in repeatedly selecting a single access path.  To
counter this effect, an alternative would be to define the quality of an access plan as a measure of {\em diversity} \cite{hui2015hear}. For example, we might prefer to equally distribute the budget
across access paths.  However, some access paths may be very uninformative / inaccurate, and optimizing diversity alone
will waste budget. Instead, we use the joint \emph{information gain}
$\text{IG}(Y;S)$ of the task variable $Y$ in our model and an access plan $S$ as a measurement of plan quality as well as an objective function for our optimization
scheme. Formally, this is is defined as:
\begin{equation}\label{eq.informationgain} 
\text{IG}(Y;S)=H(Y) - H(Y|S)
\end{equation}
An access plan $S$ determines how many variables $X$  to choose from each
access path $Z_i$. $\text{IG}(Y;S)$ measures the entropy reduction (as measure of
uncertainty) of the task variable $Y$ after an access plan $S$ is observed. At the beginning, selecting from the most
accurate access paths provides the highest uncertainty reduction. However, if better access paths are
exhausted (\emph{i.e.}, accessed relatively often), asking on less accurate ones reduces the
entropy more than continuing to ask on previously explored paths. This situation reflects the way how
information gain explores diversity and
increases the prediction confidence if evidence is retrieved from
independent channels. Based on this analysis, information gain naturally trades accuracy and diversity. While plans with high information gain do
exhibit diversity, this is only a means for achieving high predictive
performance.

\smallsection{Information gain computation}
\label{sec.igcomputation}
The computation of the conditional entropy $H(Y|S)$ as part of information gain in Equation~\ref{eq.informationgain} is
a difficult problem, as full calculation requires enumerating all possible
instantiations of the plan. Formally, the conditional
entropy can be computed as:
\begin{equation}\label{eq.conditionalH}
H(Y|S) = \sum_{y \in Y, x \in X_S} p(x,y) \log \frac {p(x)}{p(x,y)}
\end{equation}
$X_S$ refers to all the possible assignments that votes can take according to plan $S$.
We choose to follow the sampling approach presented in \cite{krause05near} which
randomly generates samples satisfying the access plan according to our Bayesian Network model. The
final conditional entropy will then be the average value of the conditional entropies of the generated samples.
This
method is known to provide absolute error guarantees for any desired level of
confidence if enough samples are generated. Moreover, it runs in polynomial time
if sampling and probabilistic inference can also be done in polynomial time. Both conditions are
satisfied by our model due to the tree-shaped configuration of the
Bayesian Network. They also hold for the Na\"{i}ve Bayes baselines as simpler
tree versions of the APM. 

%%%%%%%%%%%%%%%%%%%%%%%%%%%%%%%%%%%%%%%%%%%%%%%%
\smallsection{Submodularity of information gain}
Next, we derive the submodularity property of our objective function based on information gain in
Equation~\ref{eq.informationgain}. The property will then be leveraged by the greedy optimization scheme in proving
constant factor approximation bounds. A submodular function is a function that satisfies the law of diminishing returns
which means that the marginal gain of the function decreases while incrementally adding more elements to the input
set. 

Let $\mathcal{V}$ be a finite set. A set function $F: 2^\mathcal{V} \rightarrow \mathbb{R}$ is submodular if $F(S \cup
\{v\}) - F(S) \geq F(S' \cup \{v\}) - F(S')$ for all $S \subseteq S' \subseteq \mathcal{V}$, $v \centernot\in S'$. For
our model, this intuitively means that collecting a new vote from the crowd adds more information when few votes have
been acquired rather than when many of them have already been collected. While information gain is non-decreasing and
non-negative, it may not be submodular for a general Bayesian Network. Information gain can be shown to be submodular
for the Na\"{i}ve Bayes Model for Access Paths (NBAP) in Figure~\ref{fig:NBAP} by applying the results from
\cite{krause05near}. Here, we prove its submodularity property for the APM Bayesian Network shown in
Figure~\ref{fig:NetExample}. Theorem~\ref{theorem.submodularity} formally states the result and below we describe a
short sketch of the proof\footnote{\label{note1}Full proofs available at http://arxiv.org/abs/1508.01951}.

%%%%%%%%%%%%%%%%%%%%%%%%%%%%%%%%%%%%%%%%%%%%%%%%
\begin{theorem}\label{theorem.submodularity} \label{THMSUBM}
The objective function based on information gain in Equation~\ref{eq.informationgain} for the Bayesian Network Model for
Access Paths (APM) is submodular.
\end{theorem}

%%%%%%%%%%%%%%%%%%%%%%%%%%%%%%%%%%%%%%%%%%%%%%%%
\begin{proof}[Sketch of Theorem~\ref{theorem.submodularity}]
For proving Theorem~\ref{theorem.submodularity}, we consider a generic Bayesian Network with $N$ access
paths and $M$ possible worker votes on each access path. To prove
the submodularity of the objective function, we consider two sets (plans) $S \subset S'$ where $S' = S \cup \{v_j\}$ ,
\emph{i.e.}, $S'$ contains one additional vote from access path $j$ compared to $S$. Then, we consider
adding a vote $v_i$ from access path $i$ and we prove the diminishing return property of adding $v_i$ to $S'$ compared
to adding to $S$. The proof considers two cases. When $v_i$ and $v_j$ belong to different access paths, \emph{i.e.}, $i
\neq j$, the proof follows by using the property of conditional independence of votes from different access paths
given $Y$ and using the ``information never hurts" principle \cite{cover2012elements}. For the case of $v_i$ and $v_j$ belonging
to the same access path we reduce the network to an
equivalent network which contains only one access path $Z_i$ and then use the ``data processing inequality" principle
\cite{cover2012elements}.
\end{proof}

This theoretical result is of generic interest for other applications and a step
forward in proving the submodularity of information gain for more generic
Bayesian networks. 

%%%%%%%%%%%%%%%%%%%%%%%%%%%%%%%%%%%%%%%%%%%%%%%%
\subsection{Optimization scheme}
%%%%%%%%%%%%%%%%%%%%%%%%%%%%%%%%%%%%%%%%%%%%%%%%
\begin{algorithm2e}[t]
\caption{$\greedy$ Crowd Access Optimization}
\label{alg:greedy}
\footnotesize
{\bfseries Input:}  budget $B$ 

{\bfseries Output:} best plan $S_{best}$

{\bfseries Initialization:} $S_{best}=$\small{$\emptyset$}, $b = 0$

\While {$(\exists i \text{ s.t. } b \leq c_i)$}{
	$U_{best} = 0$
	
	\For{$i=1$ \KwTo $N$}{
		$S_{pure} = \text{PurePlan}(i)$
		
		\If{$c_i \leq B - b$}{
			$\Delta_{IG} = \ig(Y; S_{best} \cup S_{pure}) - \ig(Y, S_{best})$
			
			\If{$\frac{\Delta_{IG}}{c_i} > U_{best}$}{
   				$U_{best} = \frac{\Delta_{IG}}{c_i}$
   				
   				$S_{max} = S_{best} \cup S_{pure}$
   			}
		}
	}
	$S_{best} = S_{max}$
    
    $b = \text{cost}(S_{best})$
}
{\bfseries return} $S_{best}$
\end{algorithm2e}
%%%%%%%%%%%%%%%%%%%%%%%%%%%%%%%%%%%%%%%%%%%%%%%%
After having determined the joint information gain as an appropriate quality measure for a plan, the crowd access
optimization problem is to compute:
\begin{equation}
\label{eq:partproblem}
S_{best} = \argmax_{S \in \mathcal{S}} \text{IG}(Y;S) \mbox{ s.t. } \sum_{i=1}^{N} c_i \cdot S[i] \leq B 
\end{equation}
where $\mathcal{S}$ is the set of all plans. 
An exhaustive search would consider $|\mathcal{S}| =\prod_{i=1}^{N} \frac {B}{c_i}$ plans out of which the ones that are
not feasible have to be eliminated. Nevertheless, efficient approximation schemes can be constructed given
that the problem is an instance of submodular function maximization under budget constraints
\cite{krause2005note,2004-operations_sviridenko_budgeted-submodular-max}.
Based on the submodular and non-decreasing properties of information gain we devise a greedy technique in
Algorithm~\ref{alg:greedy} that incrementally finds a local approximation for the best plan. In each step, the algorithm
evaluates the benefit-cost ratio $U$ between the marginal information gain and cost for all feasible access paths. The
marginal information gain is the improvement of information gain by adding to the current best plan one pure vote from
one access path. In the worst case, when all access paths have unit cost, the computational complexity of the algorithm
is $\mathcal{O}(GN^2MB)$, where $G$ is the number of generated samples for computing information gain.

%%%%%%%%%%%%%%%%%%%%%%%%%%%%%%%%%%%%%%%%%%%%%%%%
\smallsection{Theoretical bounds of greedy optimization}
We now employ the submodularity of information gain in our
Bayesian network to prove theoretical bounds of the greedy optimization scheme. For the simple case of unit cost access paths, the
 greedy selection in Algorithm~\ref{alg:greedy} guarantees a utility of at least $(1 -
\sfrac{1}{e})$ $(=0.63)$ times the one obtained by optimal selection denoted by {\opt}
\cite{1978-_nemhauser_submodular-max}. However, the greedy selection scheme fails to provide approximation guarantees for
the general setting of varying costs \cite{khuller1999budgeted}.

Here, we exploit the following realistic property about the costs of the access
paths and allocated budget to prove strong theoretical guarantees about our Algorithm~\ref{alg:greedy}. We assume that the allocated budget is large enough
compared to the costs of the access paths. Formally stating, we assume that the cost of any access path $c_i$ is bounded
away from total budget $B$ by factor  $\gamma$ , \emph{i.e.},  $c_i \leq \gamma \cdot B \text{ } \forall i \in
\{1,\ldots, N\}$, where $\gamma \in (0,1)$. We state the theoretical guarantees of the  Algorithm~\ref{alg:greedy} in
Theorem~\ref{theorem.approximation} below\footref{note1}. 

%%%%%%%%%%%%%%%%%%%%%%%%%%%%%%%%%%%%%%%%%%%%%%%%
\begin{theorem}\label{theorem.approximation} \label{THMAPRX}
The $\greedy$ optimization in Algorithm~\ref{alg:greedy} achieves a
utility of at least $\Big(1 -
\frac{1}{e^{(1-\gamma)}}\Big)$ times that obtained by the optimal plan \opt, where
$\gamma = \max_{i \in \{1,\ldots, N\}} \frac {c_i}{B}$.
\end{theorem}
For instance, Algorithm~\ref{alg:greedy} achieves an approximation ratio of at least $0.39$ for $\gamma = 0.5$, and $0.59$ for
$\gamma = 0.10$.

%%%%%%%%%%%%%%%%%%%%%%%%%%%%%%%%%%%%%%%%%%%%%%%%
\begin{proof}[Sketch of Theorem~\ref{theorem.approximation}]
We follow the structure of the proof from
\cite{khuller1999budgeted,2004-operations_sviridenko_budgeted-submodular-max}.
The key idea is to use the fact that the budget spent by the algorithm at the end of
execution when it can not add an element to the solution is at least $(B -
\operatorname*{max}_{i \subseteq [1, \ldots, N]} c_i)$, which is lower-bounded by $B(1 - \gamma)$. This lower bound on the spent budget, along with the fact that the
elements are picked greedily at every iteration leads to the desired bounds.
%In fact, this is the reason, that for the setting where fractional elements can be added to solution, the bound is exactly $\Big(1 - \frac{1}{e}\Big)$ as total of $B$ budget is being spent in that case.
\end{proof}

These results are of practical importance in many other applications as the assumption of non-unit but bounded costs with respect to
budget often holds in realistic settings.
%%%%%%%%%%%%%%%%%%%%%%%%%%%%%%%%%%%%%%%%%%%%%%%%
%%%%%%%%%%%%%%%%%%%%%%%%%%%%%%% END OPTIMIZATION SECTION %%%%%%%%%%%%%%%%%%%%%%%%%%%%%%%%%%%%%%%%%%%%%%%%%%%%%%%%%%%%%%%%%%%%%%%%%
%%%%%%%%%%%%%%%%%%%%%%%%%%%%%%% BEGIN EXPERIMENTS SECTION %%%%%%%%%%%%%%%%%%%%%%%%%%%%%%%%%%%%%%%%%%%%%%%%%%%%%%%%%%%%%%%%%%%%%%%%%
\section{Experimental Evaluation}
\label{sec.experiments}
We evaluated our work on three real-world datasets. The main goal of the
experiments is to validate the proposed model and the optimization technique. We
compare our approach with other state of the art alternatives and results show
that leveraging diversity through the Access Path Model combined with the greedy
crowd access optimization technique can indeed improve the quality of
predictions.

\smallsection{Metrics}
The comparison
is based on two main metrics: \emph{accuracy} and \emph{negative
log-likelihood}.
Accuracy corresponds to the percentage of correct predictions. Negative log-likelihood is computed as the sum over all
test samples of the
negative log-likelihood
that the prediction is accurate. Hence, it measures not only the correctness of a model but also its ability to output
meaningful confidence. 
\begin{equation}
\text{-logLikelihood} = -\sum_{s_t} \log p(\text{prediction}=y_t|x_t)
\end{equation}
The closer a prediction is to the real outcome the lower
is its negative log-likelihood. Thus, a desirable model should offer
\emph{low} values of negative log-likelihood.
\begin{figure*}[t]
\begin{tikzpicture}[font=\small]
\definecolor{apm}{RGB}{120, 138, 229}
\definecolor{nbap}{RGB}{183, 219, 165}
\definecolor{nbi}{RGB}{128, 181, 101}
\definecolor{mv}{RGB}{203, 100, 93}
\definecolor{odds}{RGB}{242, 184, 85}
\pgfplotstableread{graphdata/medical_accuracy.txt}{\accuracy}
\begin{axis}[
	legend columns=2, 
    tick label style={font=\tiny},
    legend style={nodes=right,font=\tiny,anchor=south,at={(0.56,1.0)},
    draw=none, fill=none}, 
    ylabel shift = -0.2cm,
    xlabel shift = -0.1cm,
    ylabel={Accuracy},
    xlabel={Budget},
    ymin=0.5,
    xmin=3,
    xmax=30,
    bar width=.15cm,
    height=.53\columnwidth,
    width=.59\columnwidth,
    name=plot1
    ]   
    \addplot [mark size =1.5,mark=diamond*, ultra thick,draw=apm, apm, mark options=solid] table [x={tb}, y={apm_acc}]
    {\accuracy}; 
    \label{plot:apm}
    \addplot  [mark size =1.5,mark=*, thick,draw=nbap, nbap, mark options=solid] table [x={tb},y={nbap_acc}] {\accuracy};
    \label{plot:nbap}
    \addplot  [mark size =1.5,mark=square*, thick,draw=nbi, nbi,  mark options=solid] table [x={tb},y={nbi_acc}] {\accuracy};
    \label{plot:nbi}
    \addplot  [mark size =1.5,mark=triangle*, thick,draw=mv, mv, mark options=solid] table [x={tb},y={mv_acc}] {\accuracy};
    \label{plot:mv}
\end{axis}
\begin{axis}[
	legend columns=2, 
    tick label style={font=\tiny},
    ylabel={-logLikelihood},
    ylabel shift = -0.2cm,
    xlabel shift = -0.1cm,
    xlabel={Budget},
    ymax=60,
    xmin=3,
    xmax=30,
    bar width=.15cm,
    height=.53\columnwidth,
    width=.59\columnwidth,
    name=plot2,
    at=(plot1.right of south east), anchor=left of south west
    ]   
    \addplot [mark size =1.5,mark=diamond*, ultra thick,draw=apm, apm, mark options=solid] table [x={tb}, y={apm_logl}]
    {\accuracy}; \addplot  [mark size =1.5,mark=*, thick,draw=nbap, nbap, mark options=solid] table [x={tb},y={nbap_logl}] {\accuracy};
    \addplot  [mark size =1.5,mark=square*, thick,draw=nbi, nbi,  mark options=solid] table [x={tb},y={nbi_logl}] {\accuracy};
    \addplot  [mark size =1.5,mark=triangle*, thick,draw=mv, mv, mark options=solid] table [x={tb},y={mv_logl}] {\accuracy};
\end{axis}
\pgfplotstableread{graphdata/probsports_accuracy.txt}{\accuracy}
\begin{axis}[
	legend columns=2, 
    tick label style={font=\tiny},
    ylabel shift = -0.2cm,
    xlabel shift = -0.1cm,
    ylabel={Accuracy},
    xlabel={Budget},
    ymin=0.5,
    xmin=4,
    xmode=log,
    xmax=1000,
    bar width=.15cm,
    height=.53\columnwidth,
    width=.59\columnwidth,
    name=plot3,
    at=(plot2.right of south east), anchor=left of south west, 
    ]   
    \addplot [mark size =1.5,mark=diamond*, ultra thick,draw=apm, apm, mark options=solid] table [x={tb}, y={apm_acc}]
    {\accuracy}; \addplot  [mark size =1.5,mark=*, thick,draw=nbap, nbap, mark options=solid] table [x={tb},y={nbap_acc}] {\accuracy};
    \addplot  [mark size =1.5,mark=square*, thick,draw=nbi, nbi,  mark options=solid] table [x={tb},y={nbi_acc}] {\accuracy};
    \addplot  [mark size =1.5,mark=triangle*, thick,draw=mv, mv, mark options=solid] table [x={tb},y={mv_acc}] {\accuracy};
\end{axis}

\begin{axis}[
    tick label style={font=\tiny},
    ylabel={-logLikelihood},
    ylabel shift = -0.2cm,
    xlabel shift = -0.2cm,
    xlabel={Budget},
    scaled ticks=base 10:-2,
	legend columns=2, 
    ymax=600,
    xmin=4,
    xmode=log,
    xmax=1000,
    bar width=.15cm,
    height=.53\columnwidth,
    width=.59\columnwidth,
    name=plot4,
    at=(plot3.right of south east), anchor=left of south west,   
    ]   
    \addplot [mark size =1.5,mark=diamond*, ultra thick,draw=apm, apm, mark options=solid] table [x={tb}, y={apm_logl}]
    {\accuracy}; \addplot  [mark size =1.5,mark=*, thick,draw=nbap, nbap, mark options=solid] table [x={tb},y={nbap_logl}] {\accuracy};
    \addplot  [mark size =1.5,mark=square*, thick,draw=nbi, nbi,  mark options=solid] table [x={tb},y={nbi_logl}] {\accuracy};
    \addplot  [mark size =1.5,mark=triangle*, thick,draw=mv, mv, mark options=solid] table [x={tb},y={mv_logl}] {\accuracy};
\end{axis}
\matrix[
	matrix of nodes,
    anchor=south,
    draw = none,
    inner sep=0.2em,
    font=\tiny \sc,
    column 1/.style={anchor=base west},
    column 2/.style={anchor=base west},
    column 3/.style={anchor=base west},
    column 4/.style={anchor=base west},
  ]at(3.8cm, 3.1cm)
  {
    \ref{plot:apm} 	& APM 	&[5pt]
    \ref{plot:nbap}	& NBAP 	&[5pt]
    \ref{plot:nbi}	& NBI 	&[5pt]
    \ref{plot:mv}	& MV 	&[5pt]\\};
\matrix[
	matrix of nodes,
    anchor=south,
    draw = none,
    inner sep=0.2em,
    font=\tiny \sc,
    column 1/.style={anchor=base west},
    column 2/.style={anchor=base west},
    column 3/.style={anchor=base west},
    column 4/.style={anchor=base west},
  ]at(12.8cm, 3.1cm)
  {
    \ref{plot:apm} 	& APM 	&[5pt]
    \ref{plot:nbap}	& NBAP 	&[5pt]
    \ref{plot:nbi}	& NBI 	&[5pt]
    \ref{plot:mv}	& MV 	&[5pt]\\};
\end{tikzpicture}
\begin{tikzpicture}[font=\small]
\node at (5cm, -0.2cm) [minimum width = 95 mm, minimum height = 2 mm] {(a) MedicalQA};
\node at (13.9cm, -0.2cm)[minimum width = 95 mm, minimum height = 2 mm] {(b) ProbabilitySports (year = 2002)};
\end{tikzpicture}
\vspace{-5mm}
\caption{Accuracy and negative log-likelihood for equally distributed budget across all access paths. The negative
log-likelihood of Na\"{i}ve Bayes models deteriorates for high budget while for the APM it stays stable. NBI is
not competitive due to data sparsity.}
\label{fig:constrained_budget}
\end{figure*}
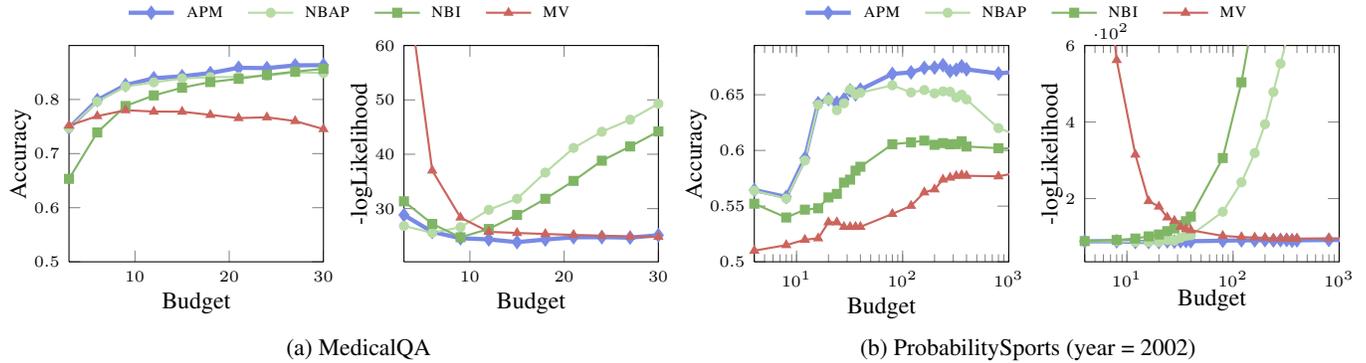
\subsection{Dataset description} All the following datasets come from real crowdsourcing tasks.
For experiments with restricted budget, we repeat the learning and prediction process via random
vote selection and k-fold cross-validation.

\smallsection{CUB-200} The dataset \cite{welinder2010caltech} was built as a large-scale data collection for attribute-based classification of bird images on Amazon Mechanical Turk (AMT). Since this is a difficult task even for experts,
the crowd workers are not directly asked to determine the bird category but whether a certain attribute is
present in the image. Each attribute (\emph{e.g.}, yellow beak) brings a piece of information
for the problem and we treat them as access paths.
The dataset contains 5-10 answers for each of the 288 available attributes. We keep the cost of all access paths equal
as there was no clear evidence of attributes that are more difficult to distinguish than others. The total number of
answers is approximately $7.5 \times 10^6$.

\smallsection{MedicalQA} We gathered 100 medical questions and
forwarded them to AMT. Workers were asked to answer the questions after reading in specific
health forums categorized as in Table \ref{tbmedicalqa} which we then design as access paths.
255 people participated in our experiment. The origin of the answer was
checked via an explanation url provided along with the answer as a sanity check. The tasks were paid equally to prevent the price
of the task to affect the quality of the answers. For experimental purposes, we assign an integer cost of
(3, 2, 1) based on the reasoning that in real life doctors are more expensive to ask, followed by patients and
common people.
\begin{table}[h!]
	\footnotesize
	\centering
	\begin{tabular}{p{0.3cm}ll}
		\midrule
		&\bf{Description}			& \bf{Forums} 						\\
		\midrule
		(1)	& Answers from doctors		& \small{www.webmd.com}				\\
			&							& \small{www.medhelp.org} 			\\\hline 
		(2) & Answers from patients		& \small{www.patient.co.uk} 			\\ 
			&							& \small{www.ehealthforum.com}		\\ \hline
		(3) & General Q\&A forum			& \small{www.quora.com}				\\
			&							& \small{www.wiki.answers.com}		\\\hline
	\end{tabular}
	\caption{Access Path Design for MedicalQA dataset.}
	\label{tbmedicalqa}
\end{table}

\smallsection{ProbabilitySports} This data is based on a crowdsourced
betting competition (www.probabilitysports.com) on NFL games.
The participants voted on the question: ``\emph{Is the home team
going to win?}'' for 250 events within a season.
There are 5,930 players in the entire dataset contributing with 1,413,534 bets. We designed the access paths
based on the accuracy of each player in the training set which does not reveal information about the testing set. Since the players' accuracy in the dataset follows a normal
distribution, we divide this distribution into three intervals where each interval corresponds to one access path (worse
than average, average, better than average). As access paths have a decreasing error
rate, we assign them an increasing cost $(2,3,4)$. 

\subsection{Model evaluation}
For evaluating the Access Path Model independently of the optimization, we first show experiments where the budget is
equally distributed across access paths. The question we want to answer here is: \emph{``How robust are the APM predictions in terms of accuracy
and negative log-likelihood?''}

\smallsection{Experiment 1: Constrained budget} 
Figure~\ref{fig:constrained_budget} illustrates the effect of data sparsity on
quality. We varied the budget and equally distributed it
across all access paths. %(\emph{e.g.} for a budget $B=30$ we randomly selected 10
%votes per access path) 
We do not show results from CUB-200 as the maximum number of votes per access path in this dataset is 5-10.

\smallsection{MedicalQA} The participation of workers in this experiment was stable, which
allows for a better error estimation.
Thus, as shown in Figure~\ref{fig:constrained_budget}(a), for
high redundancy NBI reaches comparable accuracy with the APM although
the negative log-likelihood dramatically increases. For lower budget
and high sparsity NBI cannot provide accurate results.

\smallsection{ProbabilitySports} Figure~\ref{fig:constrained_budget}(b) shows
that while the improvement of the APM accuracy over NBI and MV is stable, NBAP
starts facing the overconfidence problem while budget increases. NBI exhibits low accuracy due to
very high sparsity even for sufficient budget. Majority Vote fails to produce accurate
predictions as it is agnostic to error rates. 
\subsection{Optimization scheme evaluation}
In these experiments, we evaluate the efficiency of the greedy approximation scheme to choose high-quality
plans. For a fair comparison, we adapted the same scheme to NBI and NBAP. We will use the
following accronyms for the crowd access strategies:
OPT (optimal selection), GREEDY (greedy approximation), RND (random selection), BEST (votes from the most accurate
access path), and EQUAL (equal distribution of votes across access paths).

\smallsection {Experiment 2: Greedy approximation and diversity}
The goal of this experiment is to answer the questions:
\emph{``How close is the greedy approximation to the theoretical
optimal solution?"} and \emph{``How does
information gain exploit diversity?''}. Figure~\ref{fig:dp} shows the development of information gain for
the optimal plan, the greedily approximated plan, \rw{the equal distribution
plan}, and three pure plans that take votes only from one access path. The
quality of GREEDY is very close to the optimal plan. The third access path in
ProbabilitySports (containing better than average users)
reaches the highest information gain compared to the others.
Nevertheless, its quality is saturated for higher budget which encourages
the optimization scheme to select other access paths as well. Also, we notice that the EQUAL plan does
not reach optimal values of information gain although it maximizes diversity. 
Next, we show that the quality of predictions can be further improved
if diversity is instead planned by using information gain as an objective.
%%%%%%%%%%%%%%%%%%%%%%%%%%%%%%%%%%%%%%%%%%%%%%
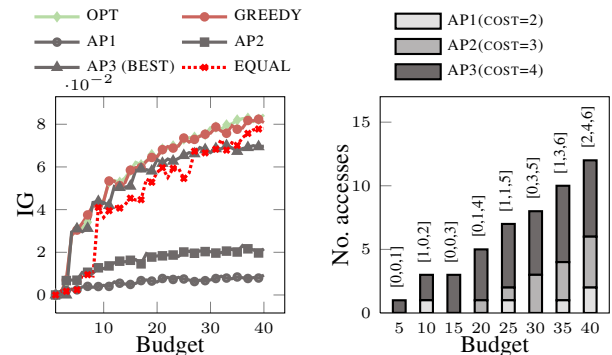
\begin{figure}[h]
\centering
\begin{tikzpicture}[font=\small]
\definecolor{dp}{RGB}{203, 100, 93}
\definecolor{opt}{RGB}{183, 219, 165}
\definecolor{ap}{RGB}{110, 105, 105}
\pgfplotstableread{graphdata/probsports_optimization.txt}{\igoptimization}
\begin{axis}[
	xlabel shift = -0.2cm,
    ylabel shift = -0.2cm,
    tick label style={font=\tiny},
    legend columns=2, 
    legend style={nodes=right,font=\tiny \sc,anchor=south,at={(0.5,1.05)},
    draw=none,fill=none}, ylabel={IG},
    xlabel={Budget},
    xmin=1,
    bar width=.10cm,
    height=.53\columnwidth,
    width=.55\columnwidth,
    mark repeat={2},
    legend entries={OPT,GREEDY,AP1,AP2,AP3 (BEST),EQUAL}      
    ]   
    \addplot [mark size =1.2,mark=diamond*, very thick,draw=opt, opt, mark options=solid] table [x={b}, y={opt}] {\igoptimization};
    \addplot  [mark size =1.2,mark=*, very thick,draw=dp, dp, mark options=solid] table [x={b},y={dp}] {\igoptimization};
    \addplot  [mark size =1.2,mark=o, draw=ap, ap, very thick, mark options=solid] table [x={b},y={ap1}] {\igoptimization};
    \addplot  [mark size =1.2,mark=square*, draw=ap, ap, very thick, mark
    options=solid] table [x={b},y={ap2}] {\igoptimization}; 
    \addplot  [mark size =1.5,mark=triangle*, draw=ap, ap, very thick, mark options=solid] table [x={b},y={ap3}] {\igoptimization};
        \addplot  [mark size =1.5,mark=x, densely dotted, draw=red, red, very
        thick, mark options=solid] table [x={b},y={equal}] {\igoptimization};
   \end{axis} 
\end{tikzpicture}
\hspace{-2pt}
\begin{tikzpicture}[font=\small]
\definecolor{dp}{RGB}{203, 100, 93}
\definecolor{opt}{RGB}{183, 219, 165}
\definecolor{ap}{RGB}{110, 105, 105}
\pgfplotstableread{graphdata/probsports_optimization_budgetdistribution.txt}{\budget}
	\begin{axis}[ybar stacked,
	xlabel shift = -0.2cm,
    ylabel shift = -0.2cm,
	ymin=0,
	ymax = 17,
	ylabel={No. accesses},
    xlabel={Budget},
    tick label style={font=\tiny},
    tickpos=left,
    xtick=data,
    legend style={nodes=right,font=\tiny \sc,anchor=south,at={(0.4,1.03)},
    draw=none,fill=none}, 
    height=.53\columnwidth,
    width=.55\columnwidth,
    bar width=.17cm,
    legend entries={AP1(cost=2),AP2(cost=3),AP3(cost=4)}, legend image code/.code={%
      \draw[#1] (0cm,-0.1cm) rectangle (0.6cm,0.1cm);
    }]
    \addplot  [draw=black, fill=ap!20, thick, mark options=solid] table [x={b},y={ap1}] {\budget};
    \addplot  [draw=black, fill=ap!50, thick, mark options=solid] table [x={b},y={ap2}] {\budget};
   	\addplot  [draw=black, fill=ap, thick, mark options=solid] table [x={b},y={ap3}] {\budget};
   	\node at (axis cs:42.5,17) [anchor=south east,font=\tiny,
   	rotate=90]{[2,4,6]}; 
   	\node at (axis cs:37.5,15) [anchor=south east,font=\tiny,
   	rotate=90] {[1,3,6]}; 
   	\node at (axis cs:32.5,13) [anchor=south
   	east,font=\tiny, rotate=90] {[0,3,5]}; 
   	\node at (axis cs:27.5,12)
   	[anchor=south east,font=\tiny, rotate=90] {[1,1,5]}; 
   	\node at (axis
   	cs:22.5,10) [anchor=south east,font=\tiny, rotate=90] {[0,1,4]}; 
   	\node at
   	(axis cs:17.5,8) [anchor=south east,font=\tiny, rotate=90] {[0,0,3]}; 
   	\node
   	at (axis cs:12.5,8) [anchor=south east,font=\tiny, rotate=90] {[1,0,2]};
   	\node at (axis cs:7.5,6) [anchor=south east,font=\tiny, rotate=90] {[0,0,1]};    	
   \end{axis}
\end{tikzpicture}
\vspace{-1mm}
\caption{Information gain and budget distribution for ProbabilitySports (year=2002). As budget
increases, GREEDY access plans exploit more than one access path.}
\label{fig:dp}
\end{figure}

\smallsection{Experiment 3: Crowd access optimization}
\label{sec:crowd_access_opt}
This experiment combines together both the model and the optimization
technique. The main question we want to answer here is: \emph{``What is the practical
benefit of greedy optimization on the APM w.r.t. accuracy and negative
log-likelihood?"}

\smallsection{CUB-200} For this dataset (Figure~\ref{fig:optimizationall}(a)) where the access path design is based on
attributes, the discrepancy between NBAP and the APM is high and EQUAL plans exhibit low quality as not all
attributes are informative for all tasks.

\smallsection{ProbabilitySports} Access Path based models (APM and NBAP) outperform MV and
NBI. NBI plans target concrete users in the competition. Hence, their accuracy for budget values less than 10 is low as
not all targeted users voted for all events. Since access paths are designed based on the
accuracy of workers, EQUAL plans do not offer a clear improvement
while NBAP is advantaged in terms of accuracy by its preference to select the
most accurate access paths. 
%%%%%%%%%%%%%%%%%%%%%%%%%%%%%%%%%%%%%%%%%%%%%%%%%%%%%%%
\begin{figure*}[t]
\centering
\begin{tikzpicture}[font=\small]
\definecolor{apm}{RGB}{120, 138, 229}
\definecolor{nbap}{RGB}{183, 219, 165}
\definecolor{nbi}{RGB}{128, 181, 101}
\definecolor{mv}{RGB}{203, 100, 93}
\definecolor{odds}{RGB}{242, 184, 85}
\pgfplotstableread{graphdata/birds_optimization_all.txt}{\optimizationall}
\begin{axis}[
	xlabel shift = -0.2cm,
    ylabel shift = -0.2cm,
    tick label style={font=\tiny},
    ylabel={Accuracy},
    xlabel={Budget},
    ymin=0.5,
    xmin=0,
    xmax=30,
    ytick={0.5, 0.55,0.6, 0.65, 0.7, 0.75, 0.8, 0.85, 0.9},
    xtick={0, 10, 20, 30, 40, 50},
    bar width=.15cm,
    height=.53\columnwidth,
    width=.59\columnwidth,
    name=plot1
    ]   
    \addplot [mark size =1.0, mark=diamond*, ultra thick,draw=apm, apm, mark options=solid] table [x={b}, y={apm_a_acc}]
    {\optimizationall}; 
    \label{plot:apm_gr}
    \addplot  [mark size =1.0, mark=*, thick,draw=nbap, nbap, mark options=solid] table [x={b},y={nbap_a_acc}] {\optimizationall};	
    \label{plot:nbap_gr}
    \addplot  [mark size =1.2, mark=x, thick,densely dotted,draw=red, red, mark
    options=solid] table [x={b},y={apm_eq_acc}] {\optimizationall};
    \label{plot:apm_eq}
\end{axis}

\begin{axis}[
	xlabel shift = -0.2cm,
    ylabel shift = -0.2cm,
    tick label style={font=\tiny},
	ylabel={-logLikelihood},
    xlabel={Budget},
    ymin=0,
    xmin=0,
    xmax=30,
    xtick={0, 10, 20, 30, 40, 50},
    bar width=.15cm,
    height=.53\columnwidth,
    width=.59\columnwidth,
    name=plot2,
    at=(plot1.right of south east), anchor=left of south west
    ]   
    \addplot [mark size =1.0, mark=diamond*, ultra thick,draw=apm, apm, mark options=solid] table [x={b},
    y={apm_a_logl}] {\optimizationall}; 
    \addplot  [mark size =1.0, mark=*, thick,draw=nbap, nbap, mark options=solid] table [x={b},y={nbap_a_logl}] {\optimizationall};
    \addplot  [mark size =1.2, mark=x, thick,densely dotted,draw=red, red, mark
    options=solid] table [x={b},y={apm_eq_logl}] {\optimizationall};
\end{axis}

\pgfplotstableread{graphdata/probsports_optimization_all.txt}{\optimizationall}
\begin{axis}[
 	xlabel shift = -0.2cm,
    ylabel shift = -0.2cm,
    tick label style={font=\tiny},
    ylabel={Accuracy},
    xlabel={Budget},
    ymin=0.5,
    xmin=0,
    xmax=30,
    xtick={0, 10, 20, 30, 40, 50},
    bar width=.15cm,
    height=.53\columnwidth,
    width=0.59\columnwidth,
    name=plot3,
    at=(plot2.right of south east), anchor=left of south west
    ]   
    \addplot [mark size =1.0, mark=diamond*, ultra thick,draw=apm, apm, mark
    options=solid] table [x={b}, y={apm_a_acc}] {\optimizationall}; 
    \addplot  [mark size =1.0, mark=*,  thick,draw=nbap, nbap, mark options=solid] table [x={b},y={nbap_a_acc}] {\optimizationall};
    \addplot  [ thick,draw=nbi, nbi, dashed, mark options=solid] table [x={b},y={nbi_a_acc}] {\optimizationall};
    \label{plot:nbi_gr}
    \addplot  [mark size =1.0, mark=square*,  ,draw=nbi, nbi,  mark options=solid] table [x={b},y={nbi_r_acc}] {\optimizationall};
    \label{plot:nbi_rnd}
    \addplot  [mark size =1.0, mark=triangle*, very thick,draw=mv, mv, mark options=solid] table [x={b},y={mv_r_acc}] {\optimizationall};
        \addplot  [mark size =1.2, mark=x, very thick,densely dotted,draw=red,
        red,mark options=solid] table [x={b},y={apm_eq_acc}] {\optimizationall};
    
    \label{plot:mv_rnd}
\end{axis}

\pgfplotstableread{graphdata/probsports_optimization_all.txt}{\optimizationall}
\begin{axis}[
	xlabel shift = -0.2cm,
    ylabel shift = -0.2cm,
    tick label style={font=\tiny},
    ylabel={-logLikelihood},
    xlabel={Budget},
    scaled y ticks=base 10:-2,
    ymin=0,
    ymax=300, 
    xmin=0,
    xmax=30,
    xtick={0, 10, 20, 30, 40, 50},
    bar width=.15cm,
    height=.53\columnwidth,
    width=.59\columnwidth,
    name=plot4,
    at=(plot3.right of south east), anchor=left of south west
    ]   
    \addplot [mark size =1.0, mark=diamond*, ultra thick,draw=apm, apm, mark options=solid] table [x={b},
    y={apm_a_logl}] {\optimizationall}; 
    \addplot  [mark size =1.0, mark=*, thick,draw=nbap, nbap, mark options=solid] table [x={b},y={nbap_a_logl}] {\optimizationall};
    \addplot  [thick,draw=nbi, nbi, dashed, mark options=solid] table [x={b},y={nbi_a_logl}] {\optimizationall};
    \addplot  [mark size =0.8, mark=square*, thick,draw=nbi, nbi,  mark options=solid] table [x={b},y={nbi_r_logl}] {\optimizationall};
    \addplot  [mark size =1.0, mark=triangle*, thick,draw=mv, mv, mark options=solid] table [x={b},y={mv_r_logl}] {\optimizationall};
        \addplot  [mark size =1.2, mark=x, thick,densely dotted,draw=red, red,
        mark options=solid] table [x={b},y={apm_eq_logl}] {\optimizationall};
\end{axis}
\matrix[
	matrix of nodes,
    anchor=south,
    draw = none,
    inner sep=0.2em,
    font=\tiny \sc,
    column 1/.style={anchor=base west},
    column 2/.style={anchor=base west},
    column 3/.style={anchor=base west},
    column 4/.style={anchor=base west},
  ]at(4.0cm, 3.0cm)
  {
    \ref{plot:apm_gr} 	& APM + GREEDY 		&[5pt]
    \ref{plot:nbap_gr}	& NBAP + GREEDY 	&[5pt]
    \ref{plot:apm_eq}	& APM + EQUAL 		&[5pt]\\};
    
\matrix[
	matrix of nodes,
    anchor=south,
    draw = none,
    inner sep=0.2em,
    font=\tiny \sc,
    column 1/.style={anchor=base west},
    column 2/.style={anchor=base west},
    column 3/.style={anchor=base west},
    column 4/.style={anchor=base west},
    column 5/.style={anchor=base west},
    column 6/.style={anchor=base west}   
  ]at(12.9cm, 3.0cm)
  {
    \ref{plot:apm_gr} 	& APM + GREEDY 		&[5pt]
    \ref{plot:nbap_gr}	& NBAP + GREEDY 	&[5pt]
    \ref{plot:apm_eq}	& APM + EQUAL		&[5pt]\\
    \ref{plot:nbi_gr} 	& NBI + GREEDY 		&[5pt]
    \ref{plot:nbi_rnd}	& NBI + RND 		&[5pt]
    \ref{plot:mv_rnd}	& MV 	+ RND 		&[5pt]\\};
\node at (4cm, -0.8cm) [minimum width = 95 mm] {(a) CUB-200 (all species)};
\node at (12.9cm, -0.8cm)[minimum width = 95 mm] {(b) ProbabilitySports (year =
2002)};
\end{tikzpicture}
\vspace{-5mm}
\caption{Crowd access optimization results for varying budget. Data sparsity and non-guaranteed votes are better handled
by the APM model also for optimization purposes, leading to improved accuracy and confidence.}
\label{fig:optimizationall}
\end{figure*}
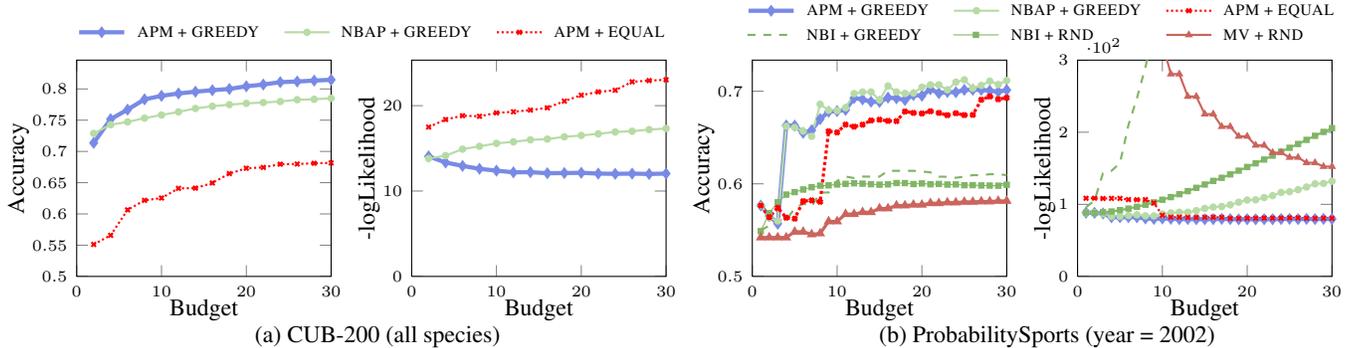
%%%%%%%%%%%%%%%%%%%%%%%%%%%%%%%%%%%%%%%%%%%%%%%%%%%%%%%%
%%%%%%%%%%%%%%%%%%%%%%%%%%%%%%%%%%%%%%%%%%%%%%%%%%%%%%%%
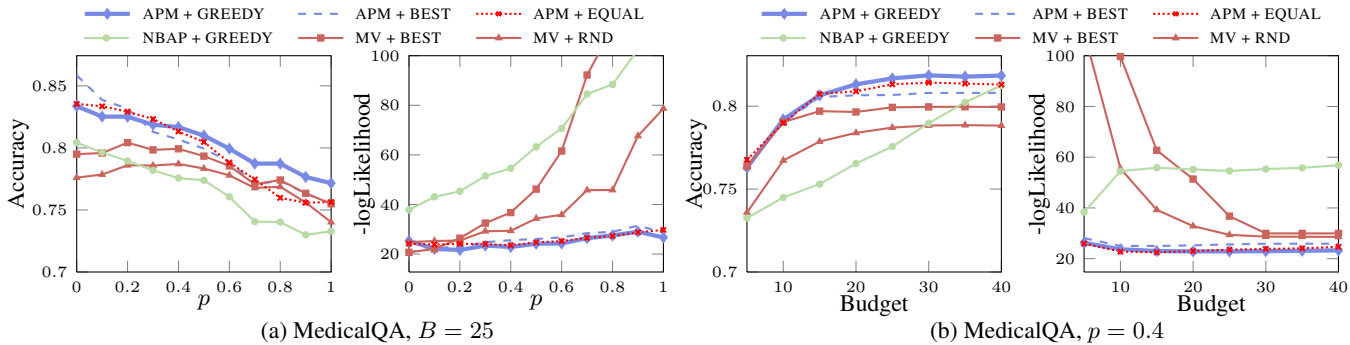
\begin{figure*}[t]
\centering
\begin{tikzpicture}[font=\small]
\definecolor{apm}{RGB}{120, 138, 229}
\definecolor{nbap}{RGB}{183, 219, 165}
\definecolor{nbi}{RGB}{128, 181, 101}
\definecolor{mv}{RGB}{203, 100, 93}
\definecolor{odds}{RGB}{242, 184, 85}
\pgfplotstableread{graphdata/medical_bias.txt}{\bias}
\begin{axis}[
 	xlabel shift = -0.2cm,
    ylabel shift = -0.2cm,
    tick label style={font=\tiny},
    legend columns=2, 
    ylabel={Accuracy},
    xlabel={$p$},
    ymin=0.7,
    xmin=0,
    xmax=1.0,
    height=.53\columnwidth,
    width=0.59\columnwidth,
    name=plot1
    ]   
    \addplot [mark size =1.0, mark=diamond*, ultra thick,draw=apm, apm, mark options=solid] table [x={p}, y={gr_apm_a}]
    {\bias}; 
    \label{plot:apm_gr}
    \addplot  [thick,draw=apm, apm, dashed, mark options=solid] table [x={p},y={best_apm_a}]
    {\bias}; 
    \label{plot:apm_best}
    \addplot  [mark size =1.0, mark=triangle*, thick,draw=mv, mv, mark options=solid] table [x={p},y={rnd_mv_a}]
    {\bias};
    \label{plot:mv_rnd}
     \addplot  [mark size =1.0, mark=square*, thick,draw=mv, mv,  mark options=solid] table [x={p},y={best_mv_a}]
    {\bias};
    \label{plot:mv_best}
     \addplot  [mark size =1.0, mark=*, thick,draw=nbap, nbap, mark options=solid] table [x={p},y={gr_nbap_a}]
    {\bias}; 
    \label{plot:nbap_gr}
     \addplot  [mark size =1.4, mark=x, thick,draw=red, red, densely
     dotted, mark options=solid] table [x={p},y={eq_apm_a}] {\bias}; 
    \label{plot:eq_apm}
\end{axis}

\begin{axis}[
 	xlabel shift = -0.2cm,
    ylabel shift = -0.3cm,
    tick label style={font=\tiny},
    legend columns=2, 
    ylabel={-logLikelihood},
    xlabel={$p$},
    ymax=100,
    xmin=0, 
    xmax=1.0,
    height=.53\columnwidth,
    width=0.59\columnwidth,
    name=plot2,
    at=(plot1.right of south east), anchor=left of south west
    ]   
    \addplot [mark size =1.0, mark=diamond*, ultra thick,draw=apm, apm, mark options=solid] table [x={p}, y={gr_apm_l}]
    {\bias}; 
    \addplot  [thick,draw=apm, apm, dashed, mark options=solid] table [x={p},y={best_apm_l}]
    {\bias}; 
    \addplot  [mark size =1.0, mark=triangle*, thick,draw=mv, mv, mark options=solid] table [x={p},y={rnd_mv_l}]
    {\bias};
     \addplot  [mark size =1.0, mark=square*, thick,draw=mv, mv,  mark options=solid] table [x={p},y={best_mv_l}]
    {\bias};
     \addplot  [mark size =1.0, mark=*, thick,draw=nbap, nbap, mark options=solid] table [x={p},y={gr_nbap_l}]
    {\bias}; 
     \addplot  [mark size =1.4, mark=x, thick,densely dotted,draw=red, red,
     mark options=solid] table [x={p},y={eq_apm_l}] {\bias}; 
\end{axis}

\pgfplotstableread{graphdata/medical_bias_fixed.txt}{\bias}
\begin{axis}[
 	xlabel shift = -0.2cm,
    ylabel shift = -0.3cm,
    tick label style={font=\tiny},
    legend columns=2, 
    ylabel={Accuracy},
    xlabel={Budget},
    ymin=0.7,
    xmin=0,
    xmin=5,
    xmax=40,
    height=.53\columnwidth,
    width=0.59\columnwidth,
    name=plot3,
    at=(plot2.right of south east), anchor=left of south west
    ]   
    \addplot [mark size =1.0, mark=diamond*, ultra thick,draw=apm, apm, mark options=solid] table [x={b}, y={gr_apm_a}]
    {\bias}; 
    \addplot  [thick,draw=apm, apm, dashed, mark options=solid] table [x={b},y={best_apm_a}]
    {\bias}; 
    \addplot  [mark size =1.0, mark=triangle*, thick,draw=mv, mv, mark options=solid] table [x={b},y={rnd_mv_a}]
    {\bias};
     \addplot  [mark size =1.0, mark=square*, thick,draw=mv, mv,  mark options=solid] table [x={b},y={best_mv_a}]
    {\bias};
     \addplot  [mark size =1.0, mark=*, thick,draw=nbap, nbap, mark options=solid] table [x={b},y={gr_nbap_a}]
    {\bias}; 
     \addplot  [mark size =1.4, mark=x, thick,densely dotted,draw=red, red,
     mark options=solid] table [x={b},y={eq_apm_a}] {\bias}; 
\end{axis}

\begin{axis}[
 	xlabel shift = -0.2cm,
    ylabel shift = -0.3cm,
    tick label style={font=\tiny},
    legend columns=2, 
    ylabel={-logLikelihood},
    xlabel={Budget},
    ymax=100,
    xmin=5,
    xmax=40,
    height=.53\columnwidth,
    width=0.59\columnwidth,
    name=plot4,
    at=(plot3.right of south east), anchor=left of south west
    ]   
    \addplot [mark size =1.0, mark=diamond*, ultra thick,draw=apm, apm, mark options=solid] table [x={b}, y={gr_apm_l}]
    {\bias}; 
    \addplot  [thick,draw=apm, apm, dashed, mark options=solid] table [x={b},y={best_apm_l}]
    {\bias}; 
    \addplot  [mark size =1.0, mark=triangle*, thick,draw=mv, mv, mark options=solid] table [x={b},y={rnd_mv_l}]
    {\bias};
     \addplot  [mark size =1.0, mark=square*, thick,draw=mv, mv,  mark options=solid] table [x={b},y={best_mv_l}]
    {\bias};
     \addplot  [mark size =1.0, mark=*, thick,draw=nbap, nbap, mark options=solid] table [x={b},y={gr_nbap_l}]
    {\bias}; 
     \addplot  [mark size =1.4, mark=x, thick,densely dotted,draw=red, red,
     mark options=solid] table [x={b},y={eq_apm_l}] {\bias}; 
\end{axis}
\matrix[
	matrix of nodes,
    anchor=south,
    draw = none,
    inner sep=0.2em,
    font=\tiny \sc,
    column 1/.style={anchor=base west},
    column 2/.style={anchor=base west},
    column 3/.style={anchor=base west},
    column 4/.style={anchor=base west},
    column 5/.style={anchor=base west},
    column 6/.style={anchor=base west}
  ]at(3.9cm, 2.9cm)
  {
    \ref{plot:apm_gr} 	& APM + GREEDY 		&[5pt]
    \ref{plot:apm_best}	& APM + BEST 	&[5pt]
    \ref{plot:apm_eq}	& APM + EQUAL			 	&[5pt]\\
    \ref{plot:nbap_gr}	& NBAP + GREEDY 	&[5pt]
    \ref{plot:mv_best} 	& MV + BEST 		&[5pt]
    \ref{plot:mv_rnd}	& MV + RND 	&[5pt]\\};
\matrix[
	matrix of nodes,
    anchor=south, 
    draw = none,
    inner sep=0.2em,
    font=\tiny \sc,
    column 1/.style={anchor=base west},
    column 2/.style={anchor=base west},
    column 3/.style={anchor=base west},
    column 4/.style={anchor=base west},
    column 5/.style={anchor=base west},
    column 6/.style={anchor=base west}
  ]at(12.9cm, 2.9cm)
  {
    \ref{plot:apm_gr} 	& APM + GREEDY 		&[5pt]
    \ref{plot:apm_best}	& APM + BEST 	&[5pt]
    \ref{plot:apm_eq}	& APM + EQUAL		&[5pt]\\
    \ref{plot:nbap_gr}	& NBAP + GREEDY 	&[5pt]
    \ref{plot:mv_best} 	& MV + BEST 		&[5pt]
    \ref{plot:mv_rnd}	& MV + RND 	&[5pt]\\};
\node at (4cm, -0.8cm) [minimum width = 95 mm] {(a) MedicalQA, $B = 25$};
\node at (12.9cm, -0.8cm)[minimum width = 95 mm] {(b) MedicalQA, $p = 0.4$};
\end{tikzpicture}
\caption{Diversity and dependence impact on optimization. As the common
dependency of workers within access paths increases, investing the whole budget
on the best access path or randomly is not efficient.}
\label{fig:bias}
\end{figure*}
%%%%%%%%%%%%%%%%%%%%%%%%%%%%%%%%%%%%%%%%%%%%%%%%%%%%%%%%

\smallsection{Experiment 5: Diversity impact}
This experiment is designed to study the impact of diversity and
conditional dependence on crowd access optimization, and finally answer the question:
\emph{``How does greedy optimization on the APM handle diversity?''}. One form of
such dependency is within access path correlation. If this correlation holds,
workers agree on the same answer. 
%In
%contrast, if there is no dependency, the answer of a worker does not imply any
%information on the expected answer of another worker within the same access
%path. 
We experimented by varying the shared dependency within the access path as
follows: Given a certain probability $p$, we decide whether a vote should follow
the majority vote of existing answers in the same access path. For example, for
$p=0.4$, 40\% of the votes will follow the majority vote decision of the previous workers and the other 
60\% will be withdrawn from the real crowd votes.

Figure~\ref{fig:bias}(a) shows that the overall quality drops when dependency is high but the Access Path Model is more
robust to it. NBAP instead, due to overconfidence, accumulates all votes into a single access path which dramatically
penalizes its quality. APM+BEST applies the APM to votes selected from the access path with the best accuracy, in
our case doctors' answers.
Results show that for $p > 0.2$, it is preferable to not only select from the best access path but to distribute the
budget according to the GREEDY scheme. Figure~\ref{fig:bias}(b) shows results from the same experiment for $p=0.4$ and
varying budget. APM+GREEDY outperforms all other methods reaching a stable quality at $B=30$ which motivates the need to design techniques that can stop the
crowdsourcing process if no new insights are possible.

\subsection{Discussion}
We presented experiments based on three different and challenging crowdsourced datasets. However, our approach and our results
are of general purpose and are not tailored to any of the datasets. The main findings are:
\begin{my_list}
\item In real-world crowdsourcing the unrealistic assumption of pairwise worker
independence poses limitations to quality assurance and increases the cost of crowdsourced solutions based on individual and majority vote models. 
\item Managing and exploiting diversity with the APM ensures quality in terms of accuracy and more
significantly negative log-likelihood. Crowd access optimization schemes on top of this perspective are practical and cost-efficient.
\item Surprisingly, access plans that combine various access paths make better predictions than plans which spend the
whole budget in a single access path.
\end{my_list}
%%%%%%%%%%%%%%%%%%%%%%%%%%%%%%% END EXPERIMENTS SECTION %%%%%%%%%%%%%%%%%%%%%%%%%%%%%%%%%%%%%%%%%%%%%%%%%%%%%%%%%%%%%%%%%%%%%%%%%
%%%%%%%%%%%%%%%%%%%%%%%%%%%%%%% BEGIN RELATED WORK SECTION %%%%%%%%%%%%%%%%%%%%%%%%%%%%%%%%%%%%%%%%%%%%%%%%%%%%%%%%%%%%%%%%%%%%%%%%%
% !TEX root =  ../document.tex
\section{Related Work}
\label{sec.relatedwork}
The reliability of crowdsourcing and relevant optimization techniques are longstanding issues for human computation
platforms. The following directions are closest to our study:

\smallsection{Quality assurance and control} One of the central works in this field is presented by
\cite{dawid1979maximum}. In an experimental design with noisy observers, the authors use an Expectation
Maximization algorithm \cite{dempster1977maximum} to obtain maximum likelihood estimates for the observer variation when
ground truth is missing or partially available. This has served as a foundation for several following contributions
\cite{ipeirotis2010quality,raykar2010learning,whitehill2009whose,zhou2012learning}, placing David and Skene's algorithm
in a crowdsourcing context and enriching it for building performance-sensitive pricing schemes. The APM model enhances
these quality definitions by leveraging the fact that the error rates of workers are directly affected by the access
path that they follow, which allows for efficient optimization.

\smallsection{Query and crowd access optimization} In crowdsourced databases, quality assurance and crowd
access optimization are envisioned as part of the query optimizer, which needs to estimate the
query plans not only according to the cost but also to their accuracy and latency. Previous work
\cite{franklin2011crowddb,marcus2011crowdsourced,parameswaran2012deco} focuses on building declarative query languages
with support for processing crowdsourced data. The proposed optimizers define the
execution order of operators in query plans and map crowdsourcable operators to micro-tasks. In our work, we
propose a complementary approach by ensuring the quality of each single operator executed by the crowd.

Crowd access optimization is similar to the expert selection problem in decision-making. However, the assumption that
the selected individuals will answer may no longer hold. Previous studies based on this assumption are
\cite{karger2011budget,ho2013adaptive,jung2013crowdsourced}. The proposed methods are nevertheless effective for task
recommendation and performance evaluation.

\smallsection{Diversity for quality} Relevant studies in management science \cite{hong2004groups,lamberson2012optimal}
emphasize diversity and define the notion of \emph{types} to refer to highly correlated forecasters. Another work that
targets groups of workers is introduced by \cite{wisdomminority}. This technique discards groups that do not prove to be
the best ones. \cite{venanzi2014community} instead, refers to groups as \emph{communities} and all of them are used for
aggregation but not for optimization. Other systems like CrowdSearcher by \cite{brambilla2014community} and CrowdSTAR by
\cite{DBLP:conf/icwe/NushiA0K15} support cross-community task allocation.
%%%%%%%%%%%%%%%%%%%%%%%%%%%%%%% END RELATED WORK SECTION %%%%%%%%%%%%%%%%%%%%%%%%%%%%%%%%%%%%%%%%%%%%%%%%%%%%%%%%%%%%%%%%%%%%%%%%%
%%%%%%%%%%%%%%%%%%%%%%%%%%%%%%% BEGIN CONCLUSIONS SECTION %%%%%%%%%%%%%%%%%%%%%%%%%%%%%%%%%%%%%%%%%%%%%%%%%%%%%%%%%%%%%%%%%%%%%%%%%
\section{Conclusion}
\label{sec.conclusion}
We introduced the Access Path Model, a novel crowd model that captures and exploits diversity as an inherent
property of large-scale crowdsourcing. This model lends itself to efficient greedy crowd access optimization. The
resulting plan has strong theoretical guarantees, since, as we prove, the information gain objective is submodular in
our model. The presented theoretical results  are of general interest and applicable to a wide range of variable selection
and experimental design problems. We evaluated our approach on three real-world crowdsourcing datasets.
Experiments demonstrate that our approach can be used to seamlessly handle critical problems in crowdsourcing
such as
quality assurance and crowd access optimization even in situations of anonymized and sparse data.

\smallsection{{\bfseries Acknowledgements}} This work was supported in part by the Swiss National Science Foundation, and
Nano-Tera.ch program as part of the Opensense II project. The authors would also like to thank Brian Galebach and Sharad
Goel for providing the Probability Sports dataset.
%%%%%%%%%%%%%%%%%%%%%%%%%%%%%%% END CONCLUSIONS SECTION %%%%%%%%%%%%%%%%%%%%%%%%%%%%%%%%%%%%%%%%%%%%%%%%%%%%%%%%%%%%%%%%%%%%%%%%%
\small
\bibliographystyle{aaai}
\bibliography{document}
%\normalsize
%%%%%%%%%%%%%%%%%%%%%%%%%%%%%%%%%%%%%%%%%%%%%%%%%%%%%%
%%%%%%%%%%%%%%%%%%%%%%%%%%%%%%%%%%%%%%%%%%%%%%%%%%%%%%
\begin{appendix}
%%%%%%%%%%%%%%%%%%%%%%%%%%%%%%%%%%%%%%%%%%%%%%%%%%%%%%
{\allowdisplaybreaks
%%%%%%%%%%%%%%%%%%%%%%%%%%%%%%%%%%%%%%%%%%%%%%%%%%%%%%
%%%%%%%%%%%%%%%%%%%%%%%%%%%%%%%%%%%%%%%%%%%%%%%%%%%%%%
%Proof sketch.

%\begin{theorem}\label{theorem.approximation} \label{THMAPRX}
%For $\alpha = 2$,  achieves a utility of at least $(\frac{e - 1}{3e}-\gamma)$ times that obtained by the optimal policy  with full knowledge of the true costs. Hereby, $\gamma$ is the ratio of the participants' largest marginal contribution  and the expected utility achieved by .
%\end{theorem}
%\noindent We show that, because of the diminishing returns property of the utility function, the stopping criteria used by the mechanism based on proportional share and using only an $\alpha$ proportion of the budget still allows the allocation of sufficiently many participants to achieve a competitive amount of utility.
%\Bigl
%\Bigr
%$\sopt$
%$\ifmax$
%${\stgreedy}$
%\sopt
%%%%%%%%%%%%%%%%%%%%%%%%%%%%%%%%%%%%%%%%%%%%%%%%%%%%%%
%%%%%%%%%%%%%%%%%%%%%%%%%%%%%%%%%%%%%%%%%%%%%%%%%%%%%%
\section{Proof of Theorem~\ref{THMSUBM}} \label{proof.submodularity}
%\adish{Some notation used below may need to be changed.}
%\bn{N - is the number of access paths}
In order to prove Theorem~\ref{theorem.submodularity}, we will consider a generic Bayesian Network for the Access Path Model
(APM) with $N$ access paths and each access path associated with $M$ possible votes from workers. Hence, we have
following set of random variables to represent this network: 
\begin{my_list}
\item[i)]$Y$ is the random variable of the crowdsourcing task.
\item[ii)]$Z:\{Z_1,\ldots, Z_i, \ldots, Z_N\}$ are the latent random variables of the $N$ access paths.
\item[iii)]$X:\{X_{ij}\text{ for } i \in [1,\ldots, N] \text{ and } j \in [1,\ldots, M]\}$ represents a set of random variables associated with
all the workers from the access paths. 
\end{my_list}

\noindent The goal is to prove the submodularity property of the set function:
\begin{align}
f(S) = IG(S; Y) \label{proof.submodularity.1}
\end{align}
\emph{i.e.}, the information gain of $Y$ and $S \subseteq X$ w.r.t to set selection $S$, earlier referred to as \emph{access plan}. 
%Given these general results,
%the  Theorem~\ref{theorem.submodularity} would follow directly as the considered network in Figure~\ref{fig:NetExample}
%is just a special case of this generic network.
 We begin by proving the following Lemma~\ref{lemma.submodularity.one} that establishes the submodularity of the
 information gain in a network with one access path (\emph{i.e.}, $N=1$), denoted as $Z_1$.

\begin{lemma}\label{lemma.submodularity.one}
The set function $f(S) = IG(S; Y)$ in Equation~\ref{proof.submodularity.1} is submodular for the Bayesian Network representing
an Access Path Model with $N=1$ access path denoted by $Z_1$, associated with $M$ workers denoted by $X:\{X_{1j} \text{
for } j \in [1,\ldots, M]$\}.
\end{lemma}

\begin{proof}[Proof of Lemma~\ref{lemma.submodularity.one}]
Figure~\ref{fig:NBAP1n} illustrates the Bayesian Network consisdered here with one access path $Z_1$. For the sake of the proof, we consider an
alternate view of the same network as shown in Figure~\ref{fig:NBAP1nAlternate}. Here, the auxiliary variable $Z_{1j}$
denotes the set of first $j$ variables associated with workers' votes from access path $Z_1$, \emph{i.e.}, $Z_{1j} =
\{X_{11}, X_{12}, \ldots, X_{1j} \}$.
This alternate view is taken from the following generative process: $Z_1$ is first sampled given $Y$, followed by
sampling of $Z_{1M}$ from $Z_1$, where $Z_{1M} = \{X_{11}, X_{12}, \ldots, X_{1M}\}$. Given $Z_{1M}$, the remaining
$Z_{1j} \text{ } \forall j \leq M$ are just subsets of $Z_{1M}$. We define set $Q:\{Z_{1j} \text{ for } j \in [1,\ldots,
M]\}$.

One crucial property we use while considering this generative process here is that all the $X_{1j}$ are just repeated
observations of same variable associated with response of a worker from $Z_1$ access path and hence they are anonymous
and ordering does not mater. Note that, querying $j$ workers from $Z_1$, \emph{i.e.} observing $S=\{X_{11} \ldots
X_{1j}\}$  is equivalent to observing $Z_{1j}$.  Given this equivalence of the two representations of
Figure~\ref{fig:NBAP1n} and Figure~\ref{fig:NBAP1nAlternate}, we now prove the submodularity of the set function $g(A)
= IG(A; Y)$ \emph{i.e.}, the information gain of $Y$ and $A \subseteq Q$ w.r.t to set selection $A$.

Note that since $Z_{1j} \subseteq Z_{1j'} \text{ } \forall \text{ } j \leq j'$, we can alternatively write down $A$ as equivalent to the singleton set given by $\{Z_{1k}\} \text{ where } k = \operatorname*{arg\,max}_{j} {Z_{1j} \in A}$.  Also note that, function $f(S)$ and $g(A)$ have one to one equivalence given by $g(A) = f(\{X_{11} \ldots X_{1k}\}) \text{ where } k = \operatorname*{arg\,max}_{j} {Z_{1j} \in A}$.

To prove submodularity of $g$, consider sets $A \subset A' \subset Q$ and an element $q \in Q \setminus A'$. Let $A \equiv \{Z_{1j}\}$,  $A' \equiv \{Z_{1j'}\} \text{ where } j' > j$ and $q = Z_{1l} \text{ where } l > j'$.
%%%%%%%%%%%%%%%%%%%%%%%%%%%%%%%%%%%%%%%%%%%%%%%%%%%%%
\begin{figure}[t]
\centering
\begin{tikzpicture}

\tikzstyle{text}=[draw=none,fill=none]
\tikzstyle{main}=[circle, minimum size = 7mm, thick, draw =black!80, node distance = 16mm]
\tikzstyle{connect}=[-latex, thick]
\tikzstyle{box}=[rectangle, draw=black!100,  dotted]
  \node[main] (Y) [label=right:$Y$] { };
  
  \node[main, fill = black!10] (Z1) [below=0.7 cm of Y,label=right:$Z_1$] {};

  \node[main] (X1jx) [below=0.8 cm of Z1,label=below:$X_{1j'}$] { };
  \node[text] (A) [right=0.4 cm of X1jx,label=below:$\ldots$] { };
  \node[text] (B) [left=0.4 cm of X1jx,label=below:$\ldots$] { };
  \node[main] (X1j) [left=0.4 cm of B,label=below:$X_{1j}$] { };
  \node[text] (C) [left=0.4 cm of X1j,label=below:$\ldots$] { };
  \node[main] (X11) [left=0.4 cm of C,label=below:$X_{11}$] { };
  
  \node[main] (X1l) [right=0.4 of A,label=below:$X_{1l}$] { };
  \node[text] (D) [right=0.4 of X1l,label=below:$\ldots$] { };
  \node[main] (X1M) [right=0.4 of D,label=below:$X_{1M}$] { };
  \node[box] (W) [below=0.55 cm of Z1, label=right:, minimum width = 85 mm,
  minimum height = 15 mm]{};
  \path
        (Y) edge [connect] (Z1)
        (Z1) edge [connect] (X1M)
        (Z1) edge [connect] (X1l)
        (Z1) edge [connect] (X1jx)
        (Z1) edge [connect] (X1j)
        (Z1) edge [connect] (X11);
                                
\end{tikzpicture}
\caption{APM Model for $N=1$ access path, associated with $M$ workers}
\label{fig:NBAP1n}
\end{figure}
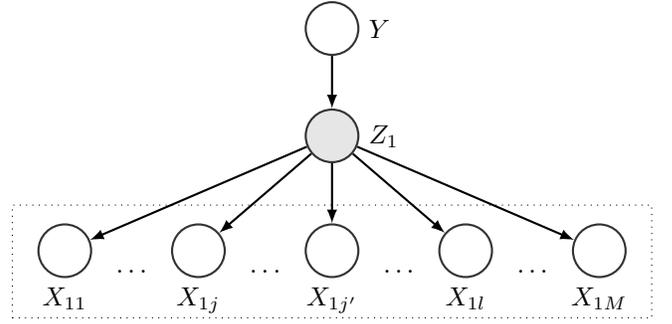
%%%%%%%%%%%%%%%%%%%%%%%%%%%%%%%%%%%%%%%%%%%%%%%%%%%%%%
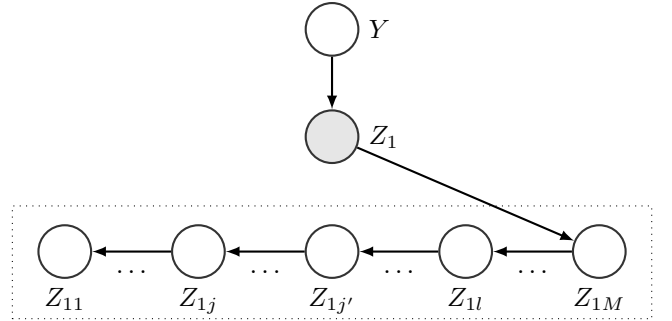
\begin{figure}[t]
\centering
\begin{tikzpicture}

\tikzstyle{text}=[draw=none,fill=none]
\tikzstyle{main}=[circle, minimum size = 7mm, thick, draw =black!80, node distance = 16mm]
\tikzstyle{connect}=[-latex, thick]
\tikzstyle{box}=[rectangle, draw=black!100,  dotted]
  \node[main] (Y) [label=right:$Y$] { };
  
  \node[main, fill = black!10] (Z1) [below=0.7 cm of Y,label=right:$Z_1$] {};

  \node[main] (Z1jx) [below=0.8 cm of Z1,label=below:$Z_{1j'}$] { };
  \node[text] (A) [right=0.4 cm of Z1jx,label=below:$\ldots$] { };
  \node[text] (B) [left=0.4 cm of Z1jx,label=below:$\ldots$] { };
  \node[main] (Z1j) [left=0.4 cm of B,label=below:$Z_{1j}$] { };
  \node[text] (C) [left=0.4 cm of Z1j,label=below:$\ldots$] { };
  \node[main] (Z11) [left=0.4 cm of C,label=below:$Z_{11}$] { };
  
  \node[main] (Z1l) [right=0.4 of A,label=below:$Z_{1l}$] { };
  \node[text] (D) [right=0.4 of Z1l,label=below:$\ldots$] { };
  \node[main] (Z1M) [right=0.4 of D,label=below:$Z_{1M}$] { };
  \node[box] (W) [below=0.55 cm of Z1, label=right:, minimum width = 85 mm,
  minimum height = 15 mm]{};
  \path
        (Y) edge [connect] (Z1)
        (Z1) edge [connect] (Z1M)
        (Z1M) edge [connect] (Z1l)
        (Z1l) edge [connect] (Z1jx)
        (Z1jx) edge [connect] (Z1j)
        (Z1j) edge [connect] (Z11);
                                
\end{tikzpicture}
\caption{APM Model for $N=1$ access path, associated with $M$ workers represented with auxiliary variables $Z_{ij}$}
\label{fig:NBAP1nAlternate}
\end{figure}
%%%%%%%%%%%%%%%%%%%%%%%%%%%%%%%%%%%%%%%%%%%%%%%%%%%%%
First, let us consider marginal utility of $q$ over $A$ denoted as $\Delta_g(q | A)$, given by:
\begin{align}
\Delta_g(q | A)	&= g(A \cup \{q\}) - g(A) \notag \\
			   	&= IG(A \cup \{q\}; Y) - IG(A; Y) \notag \\
				&= IG(\{Z_{1j}\} \cup \{Z_{1l}\}; Y) - IG(\{Z_{1j}\}; Y) \notag \\
				&= IG(\{Z_{1l}\}; Y) - IG(\{Z_{1j}\}; Y) \label{proof.submodularity.2} \\
				&= IG(Z_{1l}; Y) - IG(Z_{1j}; Y) \label{proof.submodularity.3} \\
				&= \Big(H(Y) - H(Y|Z_{1l})\Big) - \Big(H(Y) - H(Y|Z_{1j})\Big) \notag \\
				&= H(Y|Z_{1j}) - H(Y|Z_{1l}) \notag
\end{align}
Step~\ref{proof.submodularity.2} uses the fact that $\{Z_{1j}\} \cup \{Z_{1l}\}$ is simply equivalent to $\{Z_{1l}\}$ as
 $Z_{1j} \subset Z_{1l}$. Step~\ref{proof.submodularity.3} replaces singleton sets $\{Z_{1l}\}$ and $\{Z_{1j}\}$ by the
associated random variables $Z_{1l}$ and $Z_{1j}$. Now, to prove submodularity, we need to show that $\Delta_g(q | A) \geq
\Delta_g(q | A')$, given by:
\begin{align}
\Delta_g&(q | A) - \Delta_g(q | A')\notag\\ 
				&= \Big(H(Y|Z_{1j}) - H(Y|Z_{1l})\Big) - \Big(H(Y|Z_{1j'}) - H(Y|Z_{1l}) \Big) \notag \\
			   	&= H(Y|Z_{1j}) -H(Y|Z_{1j'}) \notag \\
				&= \Big(H(Y) - H(Y|Z_{1j'})\Big) - \Big(H(Y) - H(Y|Z_{1j})\Big)  \notag \\
				&= IG(Z_{1j'}; Y) - IG(Z_{1j}; Y)\notag  \\
				&\geq 0 \label{proof.submodularity.4}
\end{align}

Step~\ref{proof.submodularity.4} uses the ``data processing inequality" \cite{cover2012elements}, which states that
post-processing cannot increase information, or the mutual information gain between two random variables decreases with
addition of more intermediate random variables in the unidirectional network considered in
Figure~\ref{fig:NBAP1nAlternate}.
\end{proof}
%%%%%%%%%%%%%%%%%%%%%%%%%%%%%%%%%%%%%%%%%%%%%%%%%%%%%
Next, we use the result of Lemma~\ref{lemma.submodularity.one} to prove the results for generic networks with $N$ access paths.
% and hence proving Theorem~\ref{theorem.submodularity}.
\begin{proof}[Proof of Theorem~\ref{theorem.submodularity}]
We now consider a generic Bayesian Network for the Access Path Model (APM) with $N$ access paths and each access path associated with $M$ possible votes from workers. Again taking the alternate view as illustrated in Figure~\ref{fig:NBAP1nAlternate}, we define auxilliary variables $Z_{ij}$ denoting a set of first $j$ variables associated with workers' votes from access path $Z_i$, \emph{i.e.}, $Z_{ij} = \{X_{i1}, X_{i2}, \ldots, X_{ij} \}$. As before, we define set $Q:\{Z_{ij} \text{ for } i \in [1,\ldots, N] \text{ and } j \in [1,\ldots, M]\}$. The goal is to prove the submodularity over the set function $g(A) = IG(A; Y)$ \emph{i.e.}, the information gain of $Y$ and $A \subseteq Q$ w.r.t to set selection $A$. 

We define $Q_i:\{Z_{ij} \text{ for } j \in [1,\ldots, M]\} \text{ } \forall \text{ } i \in [1,\ldots, N]$, and hence we can write $Q = \cup_{i=1}^{N} Q_i$. We can similarly write $A = \cup_{i=1}^{N} A_i$ where  $A_i = A \cap Q_i$. We denote complements of $A_i$ and $Q_i$ as $A_i^c$ and $Q_i^c$ respectively, defined as follows: $Q_i^c = Q \setminus Q_i$ and $A_i^c = A \cap Q_i^c$. 

To prove the submodularity property of $g$, consider two sets $A \subset Q$, and $A' = A \cup \{s\}$, as well as an element $q \in Q
\setminus A'$. Let $q \in Q_i$. We consider following two cases:
%%%%%%%%%%%%%%%%%%%%%%%%%%%%%%%%%%%%%%%%%%%%%%%%%%%%%%

\smallsection{Case i)} $s \in Q_i$ ($q$ and $s$ belong to the same access path.)

\noindent Note that, we can write $A = A_i \cup A_i^c$ and $A' = A'_i \cup A_i^c$, as $A$ and $A'$ differ only along
access path $i$. Also, let us denote a particular realization of the variables in set $A_i^c$ by $a_i^c$. The key idea
that we use is that for a given realization of $A_i^c$, the generic Bayesian Network with $N$ access paths can be
factorized in a similar way as with just one access path (Figure~\ref{fig:NBAP1nAlternate}), when computing the marginal gains of $q$ over $A_i$ and $A_i \cup \{s\}$.
%, as long as we further analyze and condition only on variables related to $i$ access path (i.e. $Z_i$ and $Q_i$).

%%P(Y, A
%\begin{align}
%P(Y, Z_i, A_i) = P(Y) \cdot P(Z_i | Y) \cdot P(A_i | Z_i)
%\end{align}
%
%\begin{align}
%P(Y, Z, A|a_i^c) &= P(Y|a_i^c) \cdot \prod_{r=1}{m} P(Z_r | Y, a_i^c) \cdot \prod_{r=1}{m} P(A_i | Z_i, a_i^c) \\
%				 &= P(Y|a_i^c) \cdot \prod_{r=1}{m} P(Z_r | Y, a_r) \cdot \prod_{r=1}{m} P(A_i | Z_i, a_r) \\
%				 &= P(Y|a_i^c) \cdot P(Z_i | Y) \cdot P(A_i | Z_i) \cdot \prod_{r \in [1 \ldots m] \setminus {i}} P(Z_r | Y, a_r) \cdot \prod_{r \in [1 \ldots m] \setminus {i}} P(a_r| Z_r) \\
%\end{align}

Again, we need to show $\Delta_g(q | A) \geq \Delta_g(q | A')$; given by:
\begin{align}
\Delta_g(q | A) &- \Delta_g(q | A')\notag\\ 
				&= \Delta_g(q | A_i \cup A_i^c) - \Delta_g(q | A'_i \cup A_i^c) \notag \\
				&= \mathbb{E}_{a_i^c} \Big( \Delta_g(q | A_i, a_i^c) - \Delta_g(q | A'_i,  a_i^c) \Big) \label{proof.submodularity.9} \\
				&\geq 0 \label{proof.submodularity.10}
\end{align}

Step~\ref{proof.submodularity.9} considers expectation over all the possible realizations of random variables in
$A_i^c$. Step~\ref{proof.submodularity.10} uses the result of Lemma~\ref{lemma.submodularity.one} as this network for a
given realization of $A_i^c$ has the same characteristics as a single access
path network where information gain is submodular. Hence, each term inside the
expectation is non-negative, proving therefore the desired result.

Next, we consider the other case when $q$ and $s$ belong to different access paths.
%%%%%%%%%%%%%%%%%%%%%%%%%%%%%%%%%%%%%%%%%%%%%%%%%%%%%%

\smallsection{Case ii)} $s \in Q_i^c$ ($q$ and $s$ belong to different access paths.)
\noindent First, let us consider marginal utility of $q$ over $A$ denoted as $\Delta_g(q | A)$, given by:
\begin{align}
&\Delta_g(q | A)	= g(A \cup \{q\}) - g(A) \notag \\
			   	&= IG(A \cup \{q\}; Y) - IG(A; Y) \notag \\
				&= \Big(H(A \cup \{q\}) - H(A \cup \{q\}|Y)\Big) - \Big(H(A) - H(A|Y)\Big) \notag \\
				&= \Big(H(A \cup \{q\}) - H(A)\Big) - \Big(H(A \cup \{q\}|Y) - H(A|Y)\Big) \notag \\
				&= H(q|A) - H(q|A; Y) \label{proof.submodularity.5} \\
				&= H(q|A) - H(q|A_i; Y) \label{proof.submodularity.6}
\end{align}
%&= (H(\{q\}|A) - (H(\{q\}|A_i; Y)  

Step~\ref{proof.submodularity.5} simply replaces the singleton set $\{q\}$ with the random variable $q$. Step~\ref{proof.submodularity.6} uses the fact that $A = A_i \cup A_i^c$ and the conditional independence of $q$ and $A_i^c$ given $Y$.

Now, to prove submodularity, we need to show $\Delta_g(q | A) \geq \Delta_g(q | A')$, given by:
\begin{align}
&\Delta_g(q | A) - \Delta_g(q | A')\notag\\ 
				&= \Big(H(q|A) - H(q|A_i, Y)\Big) - \Big(H(q|A') - H(q|A_i, Y)\Big) \label{proof.submodularity.7} \\
				&= H(q|A)  - H(q|A')  \notag \\
				&\geq 0 \label{proof.submodularity.8}
\end{align}
Step~\ref{proof.submodularity.7} uses the conditional independence of  $q$ and $A_i^c$ given $Y$. Note that a crucial property used in this step is that $s \in A_i^c$ for this case. Step~\ref{proof.submodularity.8} follows from the ``information never hurts" principle \cite{cover2012elements} thus proving the desired result and completing the proof. 
\end{proof}
%%%%%%%%%%%%%%%%%%%%%%%%%%%%%%%%%%%%%%%%%%%%%%%%%%%%%%
%%%%%%%%%%%%%%%%%%%%%%%%%%%%%%%%%%%%%%%%%%%%%%%%%%%%%%
\section{Proof of Theorem~\ref{THMAPRX}} \label{proof.approximation}
%\adish{Some notation used below may need to be changed.}

\begin{proof}[Proof of Theorem~\ref{theorem.approximation}]
In order to prove Theorem~\ref{theorem.approximation}, we first consider a general submodular set function and prove the
approximation guarantees for the greedy selection scheme under the assumption that the cost to budget ratio is bounded
by $\gamma$.

Let $V$ be a collection of sets 
%\bn{let's use $\mathcal{S}$ for the collection of sets, W and w are already used earlier}
and consider a monotone, non-negative, submodular set function $f$ defined over $V$ as $f:2^V \rightarrow \Real$. Each element $v \in V$ is associated with a non-negative cost $c_v$. The budgeted optimization problem can be cast as:
\begin{align*}
S^* &= \operatorname*{arg\,max}_{S \subseteq V} f(S) \text{ subject to } \sum_{s \in S} c_s \leq B 
\end{align*}
Let $S^{\opt}$ be the optimal solution set for this maximization problem, which is intractable to compute
\cite{1998-_feige_threshold-of-ln-n}. Consider the generic  \greedy selection algorithm given by
Algorithm~\ref{alg:greedygeneral} and let $S^{\greedy}$ be the set returned by this algorithm.
%%%%%%%%%%%%%%%%%%%%%%%%%%%%%%%%%%%%%%%%%%%%%%%%
\begin{algorithm2e}[t]
\caption{\greedy for general submodular function}
\label{alg:greedygeneral}
\footnotesize
{\bfseries Input:} budget $B$, set $V$, function $f$

{\bfseries Output:} set $S^{\greedy}$

{\bfseries Initialization:} set $S =$\small{$\emptyset$}, iterations $r=0$, size $l=0$ 

\While{$V \neq \emptyset$}{

	$v^* = \operatorname*{arg\,max}_{v \subseteq V} \big(\frac{f(S \cup {v}) - f(S)}{c_v}\big)$
	
	\If{$c(S) + c_{v^*} \leq B$}{
	
		$S = S \cup \{v^*\}$
		
		$l=l+1$
		
	}
	
	$V = V \setminus \{v^*\}$
	
	$r = r+1$   
	 
}

$S^{\greedy} = S$

{\bfseries return} $S^{\greedy}$
\end{algorithm2e}
%%%%%%%%%%%%%%%%%%%%%%%%%%%%%%%%%%%%%%%%%%%%%%%%
We now analyze the performance of $\greedy$ and start by closely following the proof structure of 
\cite{khuller1999budgeted,2004-operations_sviridenko_budgeted-submodular-max}. Note that every iteration of the
Algorithm~\ref{alg:greedygeneral} can be classified along two dimensions: i) whether a selected element $v^*$ belongs to
$S^{\opt}$ or not, and ii) whether $v^*$ gets added to set $S$ or not. First, let us consider the case when $v^*$ belongs to
$S^{\opt}$, however was not added to $S$ because of violation of budget constraint. Let $r$ be the total iterations
of the algorithm so far, and $l$ be the size of $S$ at this iteration. We can renumber the elements of $V$ so that $v_i$
is the $i^{th}$ element added to $S$ for $i \in [1,\ldots, l]$ and $v_{l+1}$ is the first element from $S^{\opt}$
selected by the algorithm that could not be added to $S$. Let $S_i$ be the set obtained when first $i$ elements have
been added to $S$. Also, let $c(S)$ denote $\sum_{s \in S} c_s$. By using the result of
\cite{khuller1999budgeted,2004-operations_sviridenko_budgeted-submodular-max}, the following holds:

\begin{align}
f(S_i) - f(S_{i-1}) \geq \frac{c_i}{B} \cdot \Big( f(S^{\opt}) - f(S_{i-1}) \Big) \notag
\end{align}

Using the above result, \cite{khuller1999budgeted,2004-operations_sviridenko_budgeted-submodular-max} shows the
following through induction:
% for iterations $i = 1, \ldots, l, l+1$:
\begin{align}
f(S_l) &\geq \bigg( 1 - \prod_{j=1}^{l} \Big( 1 - \frac{c_j}{B} \Big) \bigg) \cdot f(S^{\opt}) \notag \\
	   &\geq \bigg( 1 - \Big( 1 - \sum_{j=1}^{l} \frac{c_j}{B \cdot l} \Big)^l \bigg) \cdot f(S^{\opt}) \label{proof.approximation.1} \\
       &= \bigg( 1 - \Big( 1 - \frac{c(S_l)}{B \cdot l} \Big)^l \bigg) \cdot f(S^{\opt}) \label{proof.approximation.2} 
\end{align}

In Step~\ref{proof.approximation.1}, we use the property that every function of form $\bigg( 1 - \prod_{j=1}^{l} \Big( 1 -
\frac{c_j}{B} \Big) \bigg)$ achieves its minimum at $\Big( 1 - \big( 1 - \beta \big)^l \Big)$ for $\beta =
\sum_{j=1}^{l} \frac{c_j}{B \cdot l}$.

%\begin{align}
%\end{align}

Now, we will incorporate our assumption of bounded costs, \emph{i.e.}, $c_v \leq \gamma \cdot B \text{ } \forall v \in
V$, where $\gamma \in (0,1)$ to get the desired results. We use the fact that budget spent by
Algorithm~\ref{alg:greedygeneral} at iteration $r$ when it could not add an element to solution is at least $(B -
\operatorname*{max}_{v \subseteq V} c_v)$, which is lower-bounded by $B(1 - \gamma)$. Hence, the cost of greedy solution
set $c(S_l)$ at this iteration is at least $B(1 - \gamma)$. Incorporating this in Step~\ref{proof.approximation.2}, we
get:
\begin{align}
f(S_l) &\geq \bigg( 1 - \Big( 1 - \frac{(1 - \gamma)}{l} \Big)^l \bigg) \cdot f(S^{\opt}) \notag \\
       &= \bigg( 1 - \Big( 1 - \frac{1}{\eta} \Big)^{\eta \cdot (1 - \gamma)} \bigg) \cdot f(S^{\opt}) \text{ where } \eta = \frac{l}{(1 - \gamma)} \notag \\
       &\geq \Big(1 - \frac{1}{e^{(1-\gamma)}}\Big) \cdot f(S^{\opt}) \label{proof.approximation.3} 
\end{align}
This proves that the \greedy in Algorithm~\ref{alg:greedygeneral} achieves a utility of at least $\Big(1 -
\sfrac{1}{e^{(1-\gamma)}}\Big)$ times that obtained by optimal solution \opt. Given these results,
Theorem~\ref{theorem.approximation} follows directly given the submodularity properties of the considered optimization function.
\end{proof}
%%%%%%%%%%%%%%%%%%%%%%%%%%%%%%%%%%%%%%%%%%%%%%%%%%%%%%
%%%%%%%%%%%%%%%%%%%%%%%%%%%%%%%%%%%%%%%%%%%%%%%%%%%%%%
% page breaks not allowed anymore
}
\end{appendix}
\balance
\end{document}